\documentclass[lettersize,journal]{IEEEtran}
\usepackage{amsmath,amsfonts,amssymb}
\usepackage{array}
\usepackage{textcomp}
\usepackage{stfloats}
\usepackage{url}
\usepackage{amsthm}
\usepackage{verbatim}
\usepackage{graphicx}
\graphicspath{{./}}
\usepackage{multirow}

\usepackage{caption}
\usepackage{subcaption}
\usepackage{adjustbox}
\usepackage{nicefrac}   
\usepackage{float}      
\usepackage{siunitx}    
\usepackage{xcolor}     
\usepackage{booktabs}
\usepackage{microtype}  
\usepackage{hyperref}   
\usepackage[numbers]{natbib}

\newcommand{\threepanelwidth}{0.315\textwidth}
\newcommand{\mediumfloatwidth}{0.80\textwidth}
\newcommand{\smallresultwidth}{\mediumfloatwidth}
\newcommand{\widecompactwidth}{0.96\textwidth}

\newcommand{\R}{\mathbb{R}}
\newcommand{\Mcal}{\mathcal{M}}
\newcommand{\Pcal}{\mathcal{P}}
\newcommand{\calP}{\mathcal{P}}
\newcommand{\grad}{{\rm grad}}

\DeclareMathOperator*{\argmin}{\text{argmin}}

\newcommand{\iprod}[2]{\left \langle #1, #2 \right \rangle }

\newcommand{\be}{\begin{equation}}
\newcommand{\ee}{\end{equation}}

\usepackage[algo2e, vlined,ruled]{algorithm2e}

\usepackage{enumitem}

\newtheorem{thm}{Theorem}
\newtheorem{rmk}{Remark}
\newtheorem{lem}{Lemma}
\newtheorem{cor}{Corollary}

\newtheorem{assum}{Assumption}
\newtheorem{cond}{Condition}

\definecolor{mygreen}{RGB}{205, 222, 194} 

\begin{document}

\title{Retraction-Free Optimization over the Stiefel Manifold for the LoRA Fine-Tuning}

\author{Yuan Zhang, Jiang Hu, Zhijian Lai, Lin Lin, and Zaiwen Wen
    \thanks{Yuan Zhang is with Center for Data Science, Peking University, Beijing 100871, China
    Email:~zy1002@stu.pku.edu.cn.}
    \thanks{Jiang Hu is with Yau Mathematical Sciences Center, Tsinghua University, Beijing 100084, China
    Email:~jianghu@tsinghua.edu.cn.}
    \thanks{Zhijian Lai is with Beijing International Center for Mathematical Research, Peking University, Beijing 100871, China
    Email:~lai\_zhijian@pku.edu.cn.}
    \thanks{Lin Lin is with Department of Mathematics, University of California, Berkeley, CA, US. Email:~ linlin@math.berkeley.edu.}
    \thanks{Zaiwen Wen is with Beijing	International Center for Mathematical Research, Center for Machine Learning Research and Changsha Institute for Computing and Digital Economy, Peking University, Beijing, China
    Email:~wenzw@pku.edu.cn.}
}

\maketitle

\begin{abstract}
Optimization over the Stiefel manifold plays a significant role in various machine learning tasks. Existing methods either use the retraction operators, requiring costly orthonormalization for large-scale matrices, or employ landing methods that rely on careful step size selection and penalty parameter tuning. To address these challenges, we propose a retraction-free and penalty parameter-free algorithm that directly lands on the manifold. By leveraging the strongly-convex-like property of the quadratic penalty function and the proximal smoothness of the Stiefel manifold, we establish global convergence guarantees with the best-known iteration complexities under both constant and diminishing step sizes. 
Then, we reformulate the low-rank adaptation (LoRA) fine-tuning problem for large language models as a manifold optimization problem, introducing  Manifold-LoRA for geometry-accelerated adaptation. This approach employs the proposed landing technique and a carefully designed step size strategy to accelerate the training process. Numerical experiments on benchmark datasets demonstrate the efficiency and strong downstream performance of the proposed method.
\end{abstract}


\section{Introduction}
Optimization over the Stiefel manifold has attracted considerable attention in the context of machine learning, e.g., RNN \cite{arjovsky2016unitary}, batch normalization \cite{cho2017riemannian},  distributionally robust optimization \cite{chen2017robust}, and vision transformer \cite{kong2023momentum}. The mathematical formulation of this class of problems is
\begin{align}\label{prob}
\min_{X \in \mathbb{R}^{d\times r}} \quad & f(X) = \frac{1}{N} \sum_{i=1}^N f_i(X) \\
\text{s.t.} \;\; \quad & X \in {\rm St}(d,r) := \left\{ X\in \mathbb{R}^{d\times r} : X^\top X = I \right\}, \nonumber
\end{align}
where $f_i:\R^{d\times r} \rightarrow \R$ are continuously differentiable functions and integers $r \leq d$. The most popular methods for solving \eqref{prob} are retraction-based algorithms, which have been extensively studied in the context of manifold optimization \cite{absil2008optimization,wen2013feasible,hu2020brief,boumal2023introduction}. Recently, to alleviate the possible computational burden of the retraction operator, some retraction-free methods have been developed in \cite{gao2018new,gao2022orthogonalization,ablin2022fast,ablin2023infeasible,xiao2024dissolving}. Their ideas are based on a combination of the manifold geometry and a penalty function for the manifold constraint. 

Due to their scalability and parallelization efficiency, retraction-free algorithms are well suited to certain large-scale machine learning applications. In such methods, one must simultaneously control both the constraint violation and the optimality of the loss function. Noting the similarity to decentralized optimization---where the consensus gradient step size is fixed by a known value and only the loss gradient step sizes require tuning \cite{nedic2009distributed,shi2015extra,qu2017harnessing,chen2021decentralized,deng2023decentralized}---we seek to develop retraction-free algorithms with a known penalty parameter for solving \eqref{prob}.

Another motivation for studying retraction-free methods arises from its application in the fine-tuning of large language models (LLMs). Recently, LLMs have revolutionized the field of natural language processing (NLP), achieving unprecedented performance in various applications \cite{radford2019language, qin2023chatgpt}. To tailor pretrained LLMs for specific downstream tasks, the most common approach is full fine-tuning, which requires prohibitively large computational resources due to the need to adapt all model weights, hindering the deployment of large models. Parameter-efficient fine-tuning (PEFT) has gained widespread attention as it requires few trainable parameters while delivering results comparable to or even superior to full fine-tuning. This paradigm involves inserting learnable modules or designating only a small portion of weights as trainable, keeping the main model frozen \cite{houlsby2019parameter, li2021prefix, zaken2021bitfit}. Among fine-tuning methods, low-rank adaptation (LoRA) \cite{hu2021lora} has become the de facto standard among parameter-efficient fine-tuning techniques. It assumes that the change in weights lies in a {low intrinsic dimension}, thereby modeling the update $\Delta W \in \R^{d\times m}$ by two low-rank (not greater than a small integer $r$) matrices $A \in \R^{r\times m}$ and $B\in \R^{d\times r}$, i.e., $\Delta W = BA$. Since $r \ll d$, the requirements on both storage and computation are significantly reduced. Due to its decompositional nature, there is redundancy in the representation of $\Delta W$. Traditional optimization methods for LoRA are unable to exploit this redundancy, which consequently undermines the performance of the models. Instead, we reformulate LoRA fine-tuning as an optimization problem over the product of Stiefel manifolds and Euclidean spaces.\footnote{During the review of the initial version of this manuscript, several concurrent studies emerged (see, e.g., \cite{bogachev2025riemannlora,bogachev2025lora,park2025riemannian}). They adopt closely related manifold constraints to address redundancy in LoRA and use retraction-based algorithms.} Therefore, we propose an algorithmic framework called Manifold-LoRA to accelerate the fine-tuning process and enhance model performance. Moreover, by exploiting projected gradients and incorporating a parameter-free penalty, the overhead that our method incurs is relatively negligible. Our contributions are as follows:
\begin{itemize}
    \item We prove the existence of an explicit penalty parameter by establishing a strong-convexity-like condition for the nonconvex penalty problem associated with the Stiefel manifold constraint. Building on the concept of proximal smoothness for the Stiefel manifold, we then derive convergence results for retraction-free algorithms with an explicit penalty parameter in both stochastic and deterministic settings. Notably, we show that the iterates converge exactly under a constant step size in the deterministic setting, thereby improving upon the prior convergence-to-neighborhood result in \cite{ablin2022fast}. Furthermore, our explicit choice of the penalty parameter achieves a better iteration complexity for the constraint violation in the stochastic setting with a decaying step size than that reported in \cite{ablin2023infeasible}, since we use two-scale step sizes—decaying for the loss-gradient step but constant for the penalty-gradient step—whereas \cite{ablin2023infeasible} uses a single-scale decaying step size for both. Moreover, our analysis framework---motivated by decentralized optimization\cite{nedic2009distributed,deng2023decentralized}---employs a linear decay of the constraint violation and the descent property of the loss function, which contrasts with the augmented Lagrangian-based approach used in \cite{ablin2023infeasible}.

    \item Building upon the established landing theory of retraction-free and penalty parameter-free method and the AdamW framework, we propose a new method, Manifold-LoRA, which employs a carefully designed step size strategy to accelerate the training process of fine-tuning. Compared with the conventional AdamW method, we use the penalized gradient instead of the usual gradient, and the computational overhead is negligible. Numerical experiments are conducted on a wide range of NLP tasks, demonstrating the efficiency of our algorithm. Specifically, compared to vanilla LoRA, our Manifold-LoRA with half the trainable parameters delivers fast convergence and competitive downstream performance. In particular, our method converges twice as fast as baseline methods on several typical datasets, including the SQuAD 2.0 dataset and the CoLA dataset.
\end{itemize}

\subsection{Related Work} 
\paragraph{Optimization over the Stiefel manifold} 
Optimization over the Stiefel manifold has attracted lots of attention due to its broad applications. Through the use of retraction, known as the generalization of the exponential map, the Riemannian gradient descent is proposed \cite{absil2008optimization,boumal2023introduction,hu2020brief}, where all iterations lie in the manifold. When such retraction is computationally costly, the authors \cite{gao2018new} develop a retraction-free algorithm based on the augmented Lagrangian method. More recently, by defining the constraint dissolving operator and adding a sufficiently large penalty term, the authors \cite{xiao2024dissolving} convert the manifold constrained problem \eqref{prob} into an unconstrained problem and then apply unconstrained optimization algorithms. A closely related paper to ours is \cite{ablin2023infeasible}, which first connects an augmented Lagrangian-based merit function with the landing field and then establishes the convergence by exploring the descent property of the merit function. Although the parameter associated with the penalty term could be arbitrarily chosen, an additional search is still needed when conducting landing algorithms.  Inspired by the convergence of Oja's flow, a retraction-free method is developed in \cite{ablin2022fast} for the squared Stiefel manifold (i.e., $d = r$), where the landing flow consists of the projected gradient and the gradient of the penalty function. All of these methods rely on an unknown penalty parameter to ensure the convergence. This motivates us to design penalty parameter-free algorithms, which could significantly reduce the need for tuning parameters in practical implementations.

\paragraph{LoRA} There are numerous variants of LoRA aiming to improve performance or reduce memory usage. AdaLoRA \cite{zhang2023adaptive}, a well-known successor, introduces the idea of adaptively adjusting the rank of different layers by incorporating an additional vector $\boldsymbol{g}$ to serve as the diagonal of a singular value matrix. This approach leverages a revised sensitivity-based importance measure to decide whether to disable the entries in the vector $\boldsymbol{g}$ and in the matrices $A$ and $B$. A similar work, SoRA \cite{ding2023sparse}, adopts the same model architecture as AdaLoRA, but proposes a different way to update the vector $\boldsymbol{g}$ after training. This update rule is the proximal gradient of $\mathcal{L}_1$ loss, acting as a post-pruning method. Additionally, based on the idea that networks with random initialization contain subnetworks that are optimal\cite{frankle2018lottery}, VeRA is proposed in \cite{kopiczko2023vera} to reduce memory overhead.
Although LoRA has gained significant popularity and various variants have been developed, the potential for efficient training through leveraging the 
manifold geometry to reduce redundancy has not been well-explored.

\subsection{Notation}

For a matrix $X \in \R^{d\times r}$, we use $\|X\|$ to denote its Frobenius norm. For a squared matrix $A \in \R^{r \times r}$, we define ${\rm sym}(A) = \frac{A + A^\top}{2}$ and use ${\rm diag}(A) \in \R^r$ to denote its diagonal part. For two matrices $X, Y \in \R^{d\times r}$, we use $\iprod{X}{Y}:=\sum_{i=1}^d \sum_{j=1}^r X_{ij}Y_{ij}$ to denote their Euclidean inner product. For a differential function $f:\R^{d\times r} \rightarrow \R$, we use $\nabla f(X)$ to denote its usual Euclidean gradient at $X$. We define $U_{{\rm St}(d,r)}(\frac{1}{8}) = \{X \in \mathbb{R}^{d \times r} \mid {\rm dist}(X, {\rm St}(d,r)) < \frac{1}{8}\}$ and $\bar{U}_{{\rm St}(d,r)}(\frac{1}{8}) = \{X \in \mathbb{R}^{d \times r} \mid {\rm dist}(X, {\rm St}(d,r)) \leq \frac{1}{8}\}$ with ${\rm dist}(X, {\rm St}(d,r)):= \min_{Y \in {\rm St}(d,r)} \|Y - X\|$. Let $\mathbf{1}$ represent the all-ones vector.

\section{Manifold Optimization for LoRA Fine-tuning}
In this section, we begin by reformulating LoRA fine-tuning as a manifold optimization problem, incorporating an additional constraint on the matrix $B$, which serves as the basis matrix in our reformulation. We further introduce retraction operators and review classical retraction-based methods.

\subsection{Manifold Optimization Formulation of LoRA Fine-tuning}

One possible drawback in the current LoRA fine-tuning framework is that the low-rank decomposition $\Delta W$ into product $BA$ is not unique. Specifically, for any invertible matrix $C$, it holds that $BA = (BC)(C^{-1} A)$. Note that $BC$ shares the same column space with $B$. This suggests optimizing the subspace generated by $B$ instead of $B$ itself. Numerous studies in the field of low-rank optimization, e.g., \cite{boumal2011rtrmc,dai2011subspace,dai2012geometric}, investigate the manifold geometry of the low-rank decomposition and develop efficient algorithms. However, such geometry has not been explored in the LoRA fine-tuning. 

To address such redundancy (i.e., the non-uniqueness of $BA$ representations), we regard $B$ as the basis through the manifold constraint and $A$ as the coordinate of $\Delta W$ under $B$. Hence, the optimization problem can be formulated as 

\begin{align}\label{prob:man} 
\min_{A\in \mathbb{R}^{r\times m}, \;B \in \mathbb{R}^{d\times r}} \quad & \mathcal{L}(BA), \\ 
\text{s.t.} \quad \qquad & B \in {\rm St}(d,r) {\rm ~or~} B \in {\rm Ob}(d,r), \nonumber
\end{align} 
where ${\rm Ob}(d,r):=\{B \in \R^{d\times r}: {\rm diag}(B^\top B) =\mathbf{1}\}$ and $\mathcal{L}$ represents the loss function. Compared to the Stiefel manifold ${\rm St}(d,r)$, the Oblique manifold ${\rm Ob}(d,r)$ necessitates that the matrix $B$ has unit norms in its columns, without imposing requirements for orthogonality between the columns. Problem \eqref{prob:man} is an optimization problem over the product of manifolds and Euclidean spaces. 

\subsection{Retraction-based Manifold Optimization}
Manifold optimization has attracted much attention in the past few decades, as evident in works such as   \cite{absil2008optimization,hu2020brief,boumal2023introduction}. For the Stiefel manifold ${\rm St}(d,r)$, its tangent space at $X$ is denoted by $T_X {\rm St}(d,r):=\{ \xi \in \R^{d\times r}: X^\top \xi + \xi^\top X = 0\}$, which is defined as the subspace consisting of all tangent vectors. For a differentiable function $f$, the Riemannian gradient $\grad f(X) \in T_X {\rm St}(d,r)$ is the unique tangent vector satisfying
$$
    \iprod{\grad f (X)}{\xi}_X = {\rm d}f (X)[\xi], ~~ \forall \xi \in T_X{\rm St} (d, r),
$$
where $\iprod{\cdot}{\cdot}_X$ is the Riemannian metric and ${\rm d}f$ denotes the differential of function $f$.
Since ${\rm St}(d,r)$ is a submanifold embedded in $\R^{d\times r}$, 
by setting the Riemannian metric as the Euclidean metric, the Riemannian gradient of $f$ at $X$ is
$$
\grad f (X) = \nabla f (X) - X {\rm sym} (X^\top \nabla f (X)),
$$
which is the projection of $\nabla f(X)$ onto the tangent space $T_X {\rm St}(d,r)$. The normal space $N_{X}{\rm St}(d,r)$ is defined as the orthogonal complement of $T_X{\rm St}(d,r)$ in $\R^{d\times r}$. In the design of Riemannian algorithms, an essential concept  is the so-called retraction operator.  A retraction operator $\mathcal{R}$ at $X$, denoted as $\mathcal{R}_X$, is a mapping from $T_X{\rm St}(d,r)$ to ${\rm St}(d,r)$ that satisfies the following two properties:
\begin{itemize}[left=0pt]
    \item $\mathcal{R}_X(0_X) = X$ and $0_X$ is the zero element of $T_X {\rm St}(d,r)$;
    \item $\frac{\rm d}{{\rm d}t}\mathcal{R}_{X}(t\xi)\mid_{t=0} = \xi$ for any $\xi\in T_X {\rm St}(d,r)$. 
\end{itemize}
It is well-known that the retraction operator is a generalization of the exponential map \cite{absil2008optimization}. The iterative scheme of a Riemannian gradient descent method is usually given by
$$
    X_{k+1} = \mathcal{R}_{X_k} (t_k \grad f (X_k)),
$$
where $t_k > 0$ is a step size. For the Stiefel manifold ${\rm St}(d,r)$, there are several choices for the retraction $\mathcal{R}$, such as the exponential map, the Cayley transform, the QR decomposition, and the polar decomposition, see \cite{hu2020brief} for details. Among them, the Cayley transformation proposed by \cite{wen2013feasible} is popularly used. It can be expressed as, for any $\eta \in T_X {\rm St}(d,r)$,
$$
    \mathcal{R}_X^{\rm Cayley} (-\eta) = X - U\left (I_{2r} + \frac{1}{2} V^\top U\right)^{-1}V^\top X,
$$
where concatenated matrices $U :=[ (I_d-\frac{1}{2}XX^\top) \eta, X] \in \R^{d \times (2r)}$ and $V:=[X, - (I_d - \frac{1}{2}XX^\top)\eta] \in \R^{d \times (2r)}$. This needs to invert a $(2r)$-by-$(2r)$ matrix and the total computational flops are $4dr^2 + \frac{40}{3}r^3$ \cite{jiang2015framework}, which could be calculated fast for small $r$.

\section{Retraction-free and Penalty Parameter-free Optimization} \label{sec:landing}

In this section, we focus on the design of retraction-free and penalty parameter-free algorithms for solving problem \eqref{prob}. We will first present the retraction-free algorithm and then show how the penalty parameter can be explicitly determined by characterizing the landscape of the penalty function. 

\subsection{Proximal Smoothness}

The concept of proximal smoothness, as introduced by \cite{clarke1995proximal}, refers to the characteristic of a closed set whereby the nearest-point projection becomes a singleton when the point is close enough to the set. This property facilitates algorithmic and theoretical advancements by endowing nonconvex sets with convex-like structures. Specifically, for any positive real number $\gamma$, we define the $\gamma$-tube around $\mathcal{M}$ as $U_{\mathcal{M}}(\gamma): = \{X:{\rm dist}(X,\mathcal{M}) < \gamma\}$. 
We say a closed set $\mathcal{M}$ is $\gamma$-proximally smooth if the projection operator $\Pcal_{\mathcal{M}}(X):=\argmin_{Y \in \Mcal} \|Y -X\|^2$ is a singleton whenever $X\in U_{\mathcal{M}}(\gamma)$. Indeed, as stated in \cite[Corollary 4.6]{clarke1995proximal}, a closed set $\mathcal{M}$ is convex if and only if it is $\gamma$-proximally smooth for arbitrary $\gamma \in (0, \infty)$. The Stiefel manifold $\mathcal{M}={\rm St}(d,r)$ of interest is 1-proximally smooth \cite{balashov2021gradient}.

On the other hand, it is well known that for any closed convex set $\Mcal \subset \R^{d\times r}$, the projection operator $\Pcal_{\Mcal}$ is 1-Lipschitz continuous over $\R^{d\times r}$. Similarly, following the proof in \cite[Theorem 4.8]{clarke1995proximal}, we obtain the corresponding result for the Stiefel manifold: for any $X,Y \in \bar{U}_{{\rm St}(d,r)}(\frac{1}{2} )$,
\begin{equation*}
%\label{eq:lip-proj-alpha}
\left\| \Pcal_{{\rm St}(d,r)} (X) -\Pcal_{{\rm St}(d,r)} (Y)\right\| \leq 2 \|X - Y\|.
\end{equation*}
The above properties ensure the Stiefel manifold locally behaves like a convex set, serving as a useful auxiliary result for our subsequent analysis.

\subsection{Retraction-free Algorithms}

Inspired by retraction-free algorithms \cite{ xiao2024dissolving, ablin2022fast,ablin2023infeasible}, we consider the following retraction-free gradient descent method for problem \eqref{prob}:
\begin{equation}
\label{eq:grad-it} 
X_{k+1} = X_k - \alpha_k \Pcal_{T_{X_k} {\rm St}(d,r)}(g_k) - \mu X_k(X_k^\top X_k - I), 
\end{equation}
where $\alpha_k, \mu >0$ are step sizes, $g_k$ is a stochastic estimate of $\nabla f(X_k)$, e.g., mini-batch stochastic gradient, and the mapping 
$$
\Pcal_{T_{X_k} {\rm St}(d,r)}(g):= g - X_k{\rm sym}(X_k^\top g).
$$
In contrast to retraction-based algorithms, $X_k$ does not remain on ${\rm St}(d,r)$ at each iteration.
Note that when $X_k \in {\rm St}(d,r)$, $\Pcal_{T_{X_k} {\rm St}(d,r)}$ becomes the projection operator onto the tangent space. Thus, for $X_k \in {\rm St}(d,r)$ and $g_k = \nabla f(X_k)$, it holds $\Pcal_{T_{X_k} {\rm St}(d,r)}(g_k) = \grad f(X_k)$. Moreover, the term $X_k(X_k^\top X_k - I)=\nabla \varphi(X_k)$ in \eqref{eq:grad-it} is exactly the gradient of the following quadratic penalty function
$$
    \varphi (X): = \frac{1}{4}\|X^\top X - I\|^2.
$$
We present the detailed description of retraction-free algorithm in Algorithm \ref{alg:retr-free}.

\begin{figure*}[!t]
\centering
% \small
\begin{subfigure}[t]{\threepanelwidth}
    \centering
    \includegraphics[width=\linewidth]{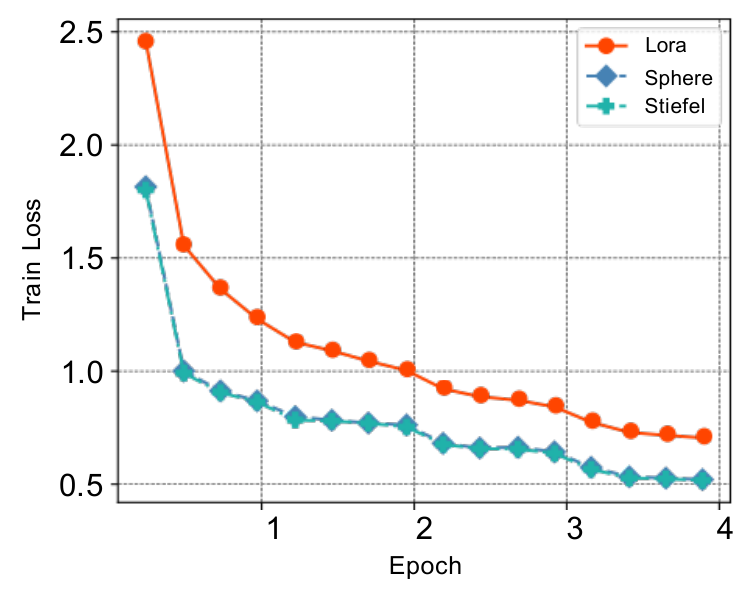}
    \caption{SQuADv2.0 Train Loss}
    \label{fig:squad-1}
\end{subfigure}
\hfill
\begin{subfigure}[t]{\threepanelwidth}
    \centering
    \includegraphics[width=\linewidth]{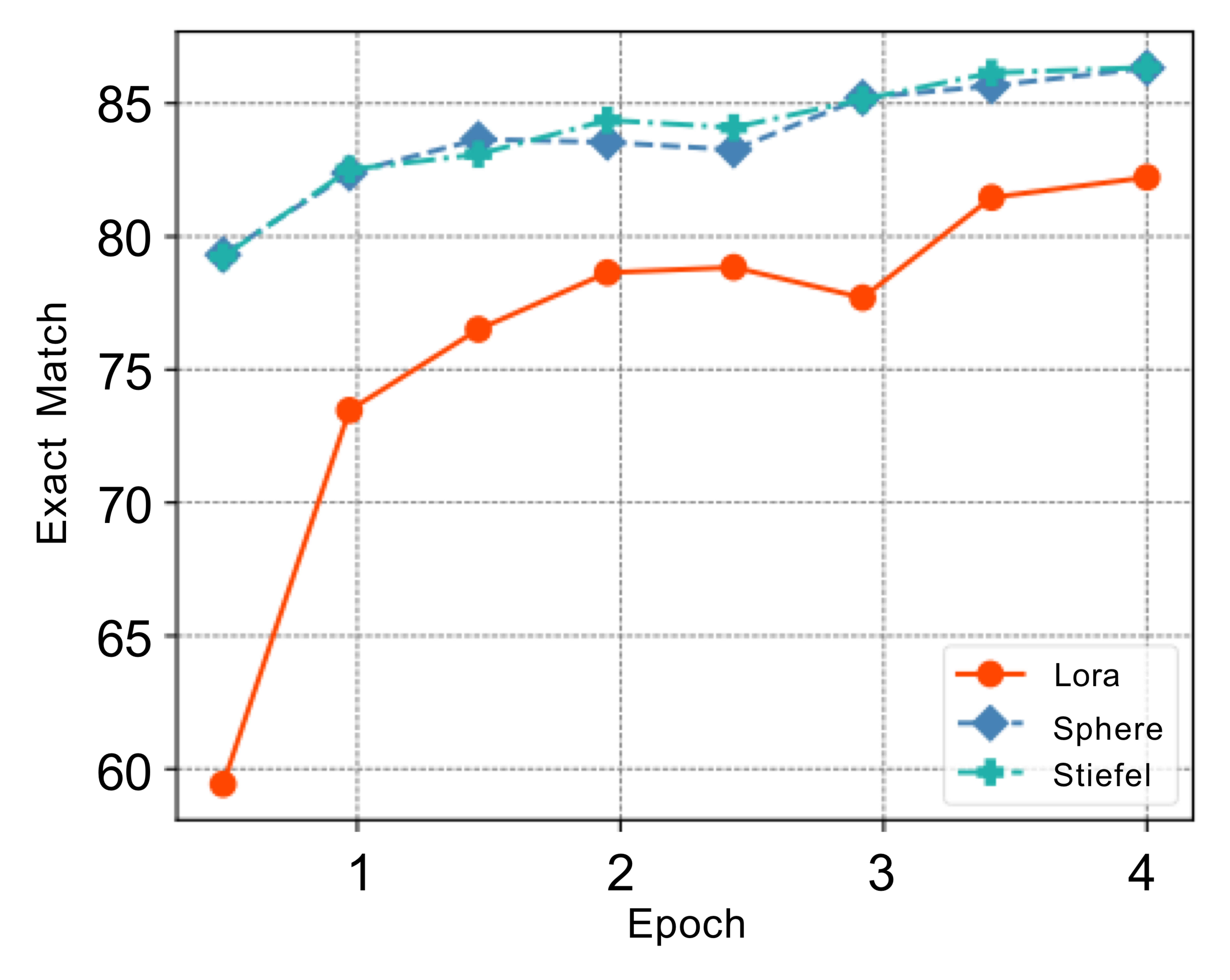}
    \caption{SQuADv2.0 Eval Exact Match}
    \label{fig:squad-2}
\end{subfigure}
\hfill
\begin{subfigure}[t]{\threepanelwidth}
    \centering
    \includegraphics[width=\linewidth]{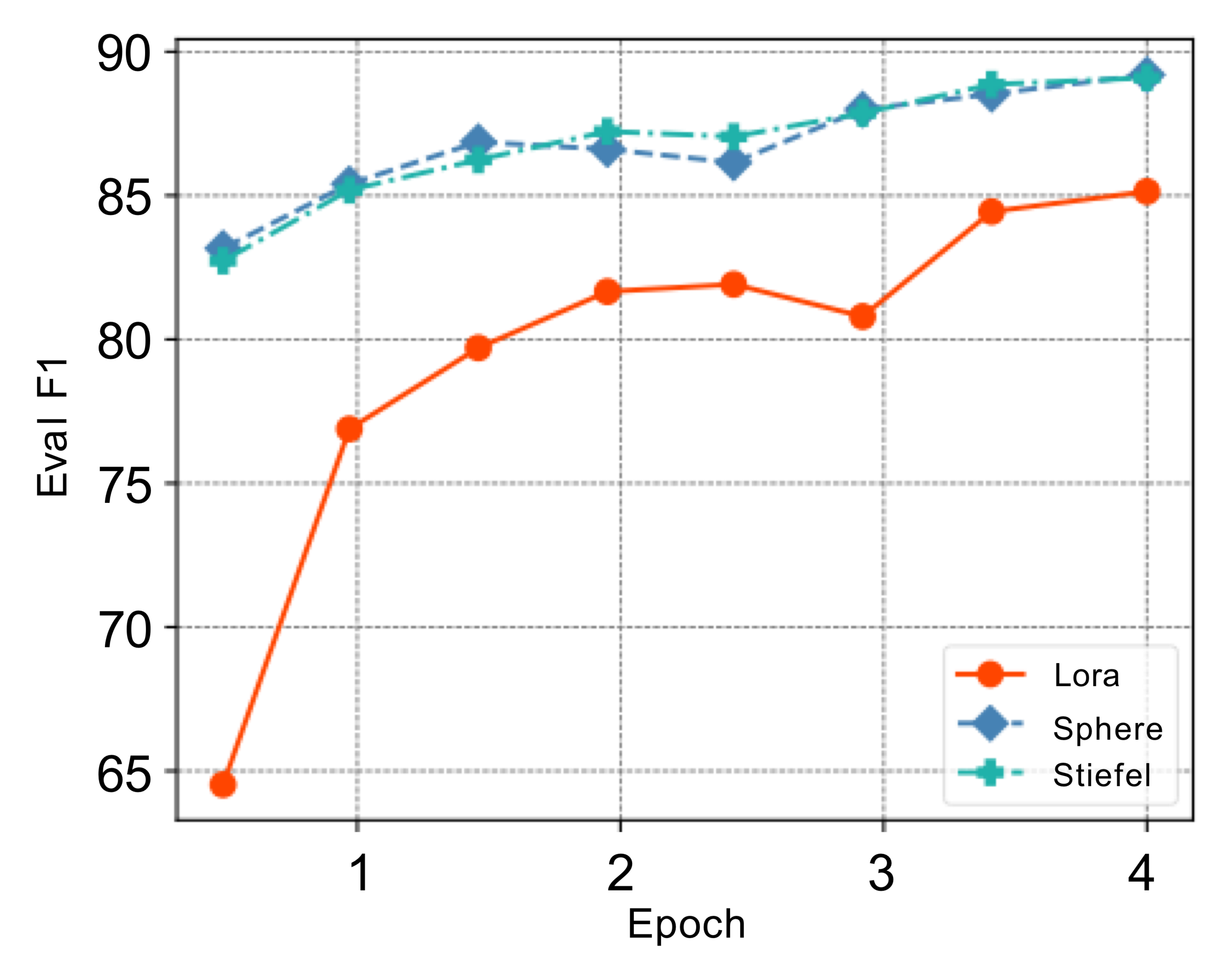}
    \caption{SQuADv2.0 Eval F1}
    \label{fig:squad-3}
\end{subfigure}
\caption{The figures compare the training loss, evaluation exact match, and evaluation F1 metrics against epochs for the SQuADv2.0 dataset. It can be clearly seen that our proposed Manifold-LoRA method almost achieves a 2x speed-up in training epochs compared to the vanilla LoRA.}
\label{fig:squadv2}
\end{figure*}
\begin{algorithm2e}[t]
\small
\caption{Retraction-free algorithm for solving \eqref{prob}}
\SetAlgoLined
\DontPrintSemicolon  
\SetKwInput{KwInput}{Input}
\KwInput{Initial point $X_0$, $\alpha_k,\mu>0$, and $k = 0$.}
\While{Stopping conditions not met}{
Compute the stochastic gradient $g_k$ of $\nabla f(X_k)$. \;
Compute
$X_{k+1} \leftarrow  X_k - \alpha_k \Pcal_{T_{X_k} {\rm St}(d,r)}(g_k) - \mu X_k(X_k^\top X_k - I).$\;
Update $k \gets k + 1$.\; %\hspace{-1em} 
}
\label{alg:retr-free}
\end{algorithm2e}

As will be shown in our theorem, the negative penalty gradient $-\nabla \varphi(X_k)$ pulls the iterate $X_{k+1}$ back to the manifold, while the use of the projected stochastic gradient $\Pcal_{T_{X_k} {\rm St}(d,r)}(g_k)$ is crucial for ensuring its asymptotic orthogonality with $\nabla \varphi(X_k)$, resulting in landing on the manifold and convergence to a stationary point simutaneously.  This differs from the usual penalty method, which optimizes $f(X) + \mu \varphi(X)$ using the update $X_{k+1} = X_k - \alpha_k g_k - \mu X_k(X_k^\top X_k - I)$, and requires $\mu \rightarrow \infty$ to guarantee the feasibility. 

A key distinction from existing works \cite{xiao2024dissolving, ablin2022fast, ablin2023infeasible} is that our approach allows for a constant step size $\mu$ in the penalty term while ensuring convergence in both deterministic and stochastic settings. Notably, we set $\mu = 1/3$ by requiring $X_1$ is not far away from ${\rm St}(d,r)$. It is worth highlighting that similar explicit choices for constraint violation have been explored in decentralized optimization \cite{nedic2009distributed, deng2023decentralized}, where a fixed step size—typically set to 1—is commonly employed for enforcing consensus constraints. The theoretical foundation supporting our approach relies on the restricted strong convexity of the penalty function $\varphi$ and the proximal smoothness of the Stiefel manifold, which together facilitate convergence guarantees in Section \ref{sec:con}.

Compared with the popularly used Cayley transformation-based retraction-type algorithms, the computational cost therein is $4dr^2 + \frac{40}{3}r^3$, which is more than twice the cost of our method at $2dr^2$ for any $r$. Moreover, retractions on the Stiefel manifold involve complex orthogonalization procedures, such as matrix inversion in the Cayley transformation, which are difficult to scale and parallelize. In contrast, the landing update \eqref{eq:grad-it} can be executed using scalable BLAS3 operations.

\subsection{Manifold-LoRA}

The retraction-free method is well-suited to address \eqref{prob:man}, simultaneously minimizing the loss function $\mathcal{L}(BA)$ and constraint violation of $B$. To control the constraint violation, we use the quadratic penalties $R_{\mathrm{St}}(B):=\|B^\top B - I\|^2$ and $R_{\mathrm{Ob}}(B):=\|{\rm diag}(B^\top B) - \mathbf{1}\|^2$ for the Stiefel manifold and Oblique manifold, respectively. By the retraction-free method, Algorithm \ref{alg:retr-free}, we use the projected gradient of the loss part instead of the Euclidean gradient. For simplicity, we write $\nabla_B \mathcal{L}=\nabla_B \mathcal{L} (BA) $, and similarly $\nabla_A \mathcal{L}$. For $B \in \mathrm{St} (d, r)$ or $B \in \mathrm{Ob} (d, r)$, the respective projected gradients are 
\begin{equation}
\label{eq:grad1} 
\Pcal_{T_{B} {\rm St} (d, r)}(\nabla_B \mathcal{L}) = \nabla_B \mathcal{L}- B{\rm sym}(B^\top \nabla_B \mathcal{L}) 
\end{equation}
and 
\begin{equation}
\label{eq:grad2} 
\Pcal_{T_{B} {\rm Ob} (d, r)}(\nabla_B \mathcal{L}) = \nabla_B \mathcal{L}- B{\rm ddiag}(B^\top \nabla_B \mathcal{L}) 
\end{equation}
where ${\rm ddiag}(Z)$ denotes $Z$ with all off-diagonal entries set to $0$. Thus, the gradients of our retraction-free method for $A$ and $B$ are $\nabla_A \mathcal{L}$ and $\Pcal_{T_{B} {\rm St} (d, r)}(\nabla_B \mathcal{L})+ \mu \nabla R_{\mathrm{St}}(B) ({\rm ~or~} \Pcal_{T_{B} {\rm Ob} (d, r)}(\nabla_B \mathcal{L})+ \mu \nabla R_{\mathrm{Ob}}(B))$. Note that $B$ and $A$ represent the basis and the coordinate of $\Delta W$, respectively. This results in different magnitudes and different Lipschitz constants of their gradient function. In fact, let $X = BA$. It follows
$$
    \nabla_A \mathcal{L} (BA) = B^\top\nabla_X \mathcal{L} (X), \quad \nabla_B \mathcal{L} (BA) = \nabla_X \mathcal{L} (X) A^\top.
$$
Then, it holds that for any $A_1,A_2 \in \mathbb{R}^{r \times m}, B_1,B_2 \in \mathbb{R}^{d \times r}$,
$$
    \begin{aligned}
    \|\nabla_A \mathcal{L} (BA_1) - \nabla_A \mathcal{L} (BA_2) \| & \leq \|B\|_2 L_g \|A_1 - A_2\|, \\
    \|\nabla_B \mathcal{L} (B_1A) - \nabla_B \mathcal{L} (B_2A) \| & \leq \|A\|_2 L_g \|B_1 - B_2\|,
    \end{aligned}
$$
where $L_g$ is the Lipschitz constant of $\nabla_X \mathcal{L}(X)$ and $\|\cdot \|_2$ represents the matrix $\ell_2$ norm (i.e., the largest singular value). Note that the step size should generally be proportional to the reciprocal of Lipschitz constant for gradient-type algorithms \cite{nocedal1999numerical,bottou2018optimization}. Hence, we schedule the learning rates for the two matrices based on their respective $\ell_2$ norms. Having prepared the above, we incorporate the SGD or AdamW optimizer \cite{loshchilov2017decoupled} with our manifold-accelerated technique to enhance the LoRA fine-tuning, as presented in Algorithm \ref{alg:manlora}, which can be seen as a generalization of Algorithm \ref{alg:retr-free} to solve optimization problems over the product of Stiefel or Oblique manifold and Euclidean space. Note that the Oblique manifold ${\rm Ob}(d,r)$ is the product of $r$ Stiefel manifolds ${\rm St}(d,1)$, which is the sphere in $\R^d$. 

\begin{figure*}[b]

\centering

\begin{subfigure}[t]{\threepanelwidth}
    \centering
    \includegraphics[width=\linewidth]{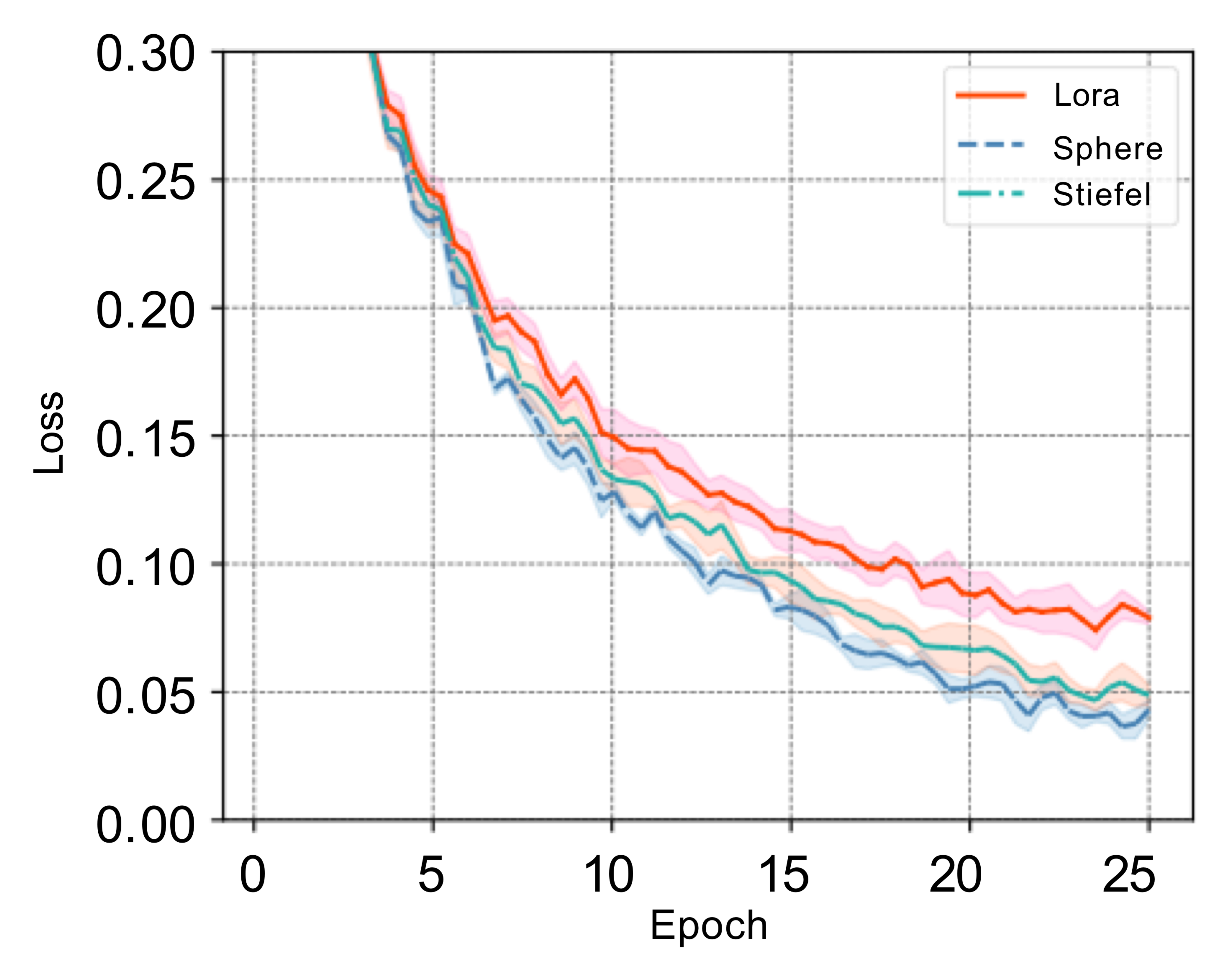}
    \caption{Loss curves on CoLA dataset.}
    \label{fig:sub-1}
\end{subfigure}
\hfill
\begin{subfigure}[t]{\threepanelwidth}
    \centering
    \includegraphics[width=\linewidth]{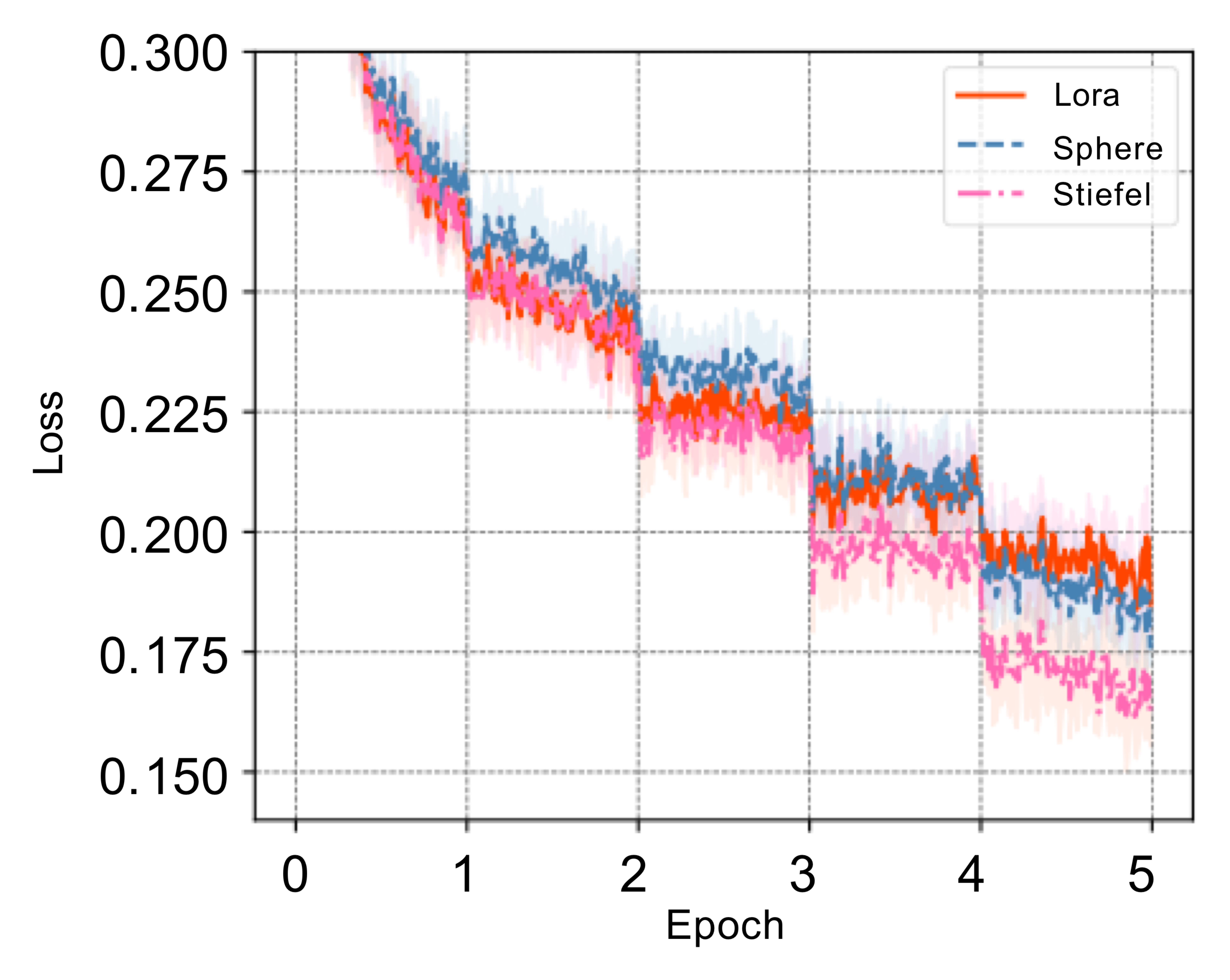}
    \caption{Loss curves on QQP dataset.}
    \label{fig:sub-2}
\end{subfigure}
\hfill
\begin{subfigure}[t]{\threepanelwidth}
    \centering
    \includegraphics[width=\linewidth]{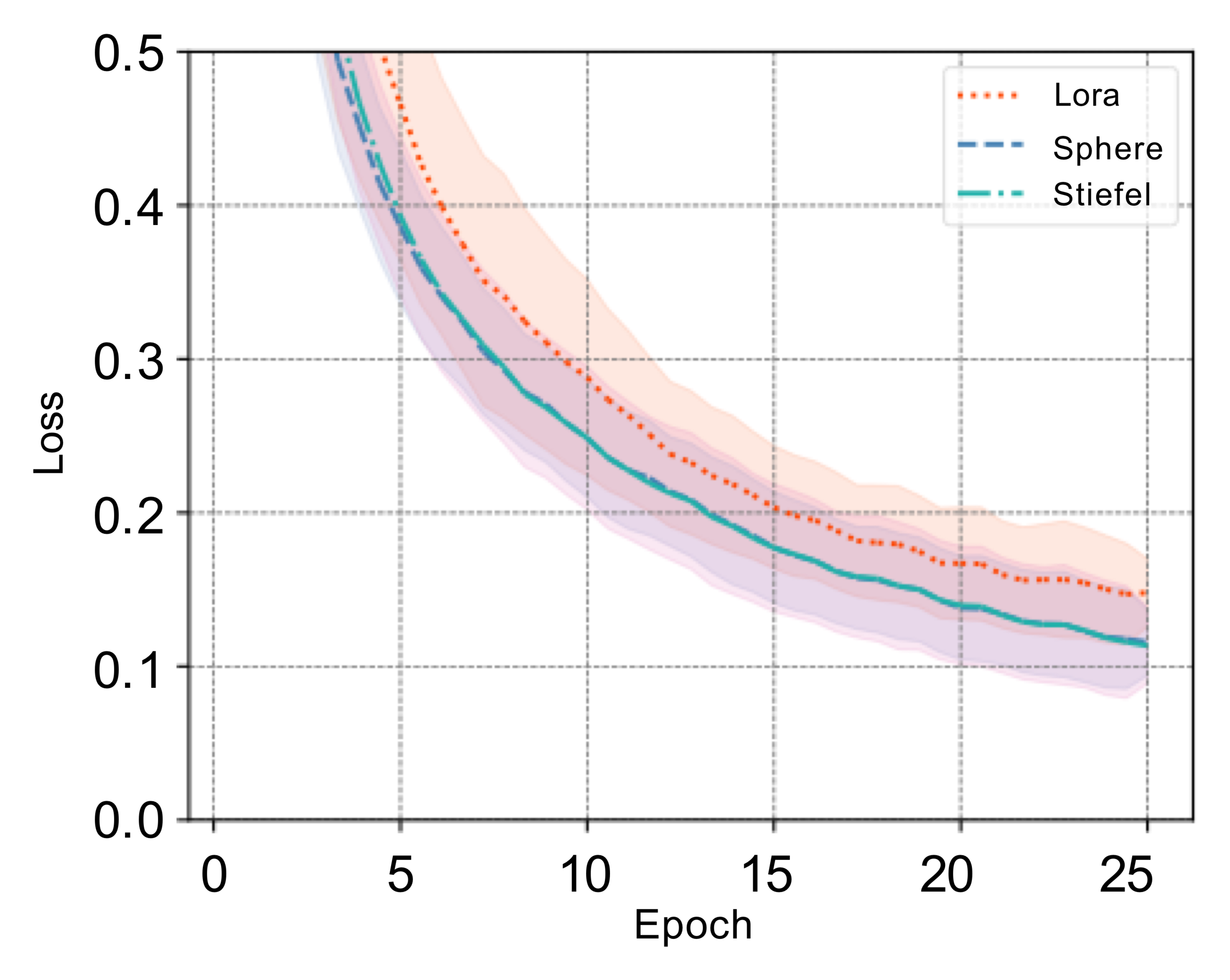}
    \caption{Loss curves on STSB dataset.}
    \label{fig:sub-3}
\end{subfigure}

\caption{The figures illustrate that both sphere constrained and Stiefel constrained manifold-LoRA achieve a faster convergence rate and attain a lower training loss within same optimization steps compared to LoRA method on three distinct datasets CoLA, QQP, STS-B.}
\label{fig:glue-images}
\end{figure*}

\begin{algorithm2e}[t]
\small
\caption{Manifold-LoRA for solving \eqref{prob:man}}
\SetAlgoLined
\DontPrintSemicolon  
\SetKwInput{KwInput}{Input} 
\KwInput{Initial $A_0, B_0$, choice \texttt{update\_type} $\in\{\texttt{SGD},\,\texttt{AdamW}\}$, $\alpha_k,\mu>0$, $\beta_1 = 0.9$, $\beta_2 = 0.999$, ${\rm ub} \geq {\rm lb} > 0$, $\epsilon = 10^{-8}$, $\lambda > 0$, $m(C_{-1})=v(C_{-1}) =0$, $k = 1$.}
\While{Stopping conditions not met}{
    Compute the stochastic gradients $g(A_k)$ and $g(B_k)$ of $\nabla_A \mathcal{L}(B_kA_k)$ and $\nabla_B \mathcal{L}(B_k A_k)$, respectively.

    Let $\hat{g}(A_k) \leftarrow g(A_k)$.

    Let $\hat{g}(B_k) \leftarrow \mathcal{P}_{T_{B_k} \mathrm{St}(d, r)}\left(g(B_k)\right)$ as in \eqref{eq:grad1}, or $\hat{g}(B_k) \leftarrow \mathcal{P}_{T_{B_k} \mathrm{Ob}(d, r)}\left(g(B_k)\right)$ as in \eqref{eq:grad2}.
    
    \For{$C \in \{A, B\}$}{
        \tcc{Projected grad update for A, B}
    
        \If{\texttt{update\_type} == \texttt{SGD}}{
            $C_{k+1} \gets C_k - \alpha_k \, \hat{g}(C_k)$\;
        }

        \If{\texttt{update\_type} == \texttt{AdamW}}{
            $m(C_{k}) \gets \beta_1 \, m(C_{k-1}) + (1 - \beta_1) \, \hat{g}(C_k)$\;
            $v(C_{k}) \gets \beta_2 \, v(C_{k-1}) + (1 - \beta_2) \, \hat{g}(C_k) \odot \hat{g}(C_k)$\;
            $\hat{m}(C_{k}) \gets \tfrac{m(C_k)}{1 - \beta_1^k}$\;
            $\hat{v}(C_{k}) \gets \tfrac{v(C_k)}{1 - \beta_2^k}$\;
            $\alpha(C_k) \gets \text{clip}(\|C_k\|_2, \text{ub}, \text{lb})$\;
            % Scheduling step sizes of A, B
            
            $C_{k+1} \gets C_k - \alpha(C_k) \, \tfrac{\hat{m}(C_{k})}{\sqrt{\hat{v}(C_k)} + \epsilon} - \lambda C_k$\;
        }

        \tcc{Penalty grad update for B}
        \If{$C = B$}{$C_{k+1} \gets C_k - \mu \nabla R_{\mathrm{St}}(C_k)({\rm or~}\nabla R_{\mathrm{Ob}}(C_k))$ 
        }
    }
    Update $k \gets k + 1$.
}
\label{alg:manlora}
\end{algorithm2e}

\section{Convergence Analysis} \label{sec:con}
In this section, we analyze the convergence of our retraction-free gradient descent method \eqref{eq:grad-it}. We first justify the choice of $\mu = \frac{1}{3}$ under an appropriate initialization. Subsequently, we establish convergence results for \eqref{eq:grad-it} in both deterministic and stochastic settings.

\subsection{Explicit Choice for the Penalty Parameter}

It is known that a large penalty parameter $\mu$ yields better feasibility \cite[Chapter 17]{nocedal1999numerical}. To make the iterative scheme \eqref{eq:grad-it} be penalty parameter-free, we need a careful investigation on the landscape of the following optimization problem:
\begin{equation}
\label{prob:feasi}  
\min_{X\in \R^{d\times r}} \;\; \varphi(X)=\frac{1}{4}\left\|X^{\top} X-I\right\|^2. 
\end{equation}
It can be easily verified that problem \eqref{prob:feasi} is nonconvex and its optimal solution set is ${\rm St}(d,r)$. The key of obtaining an explicit formula of $\mu$ is to establish certain strong convexity-type inequality and show that the gradient descent method with step size $\mu$ has linear convergence.  

For any $X \in \mathbb{R}^{d \times r}$, let us denote $\bar{X}:= \calP_{{\rm St}(d,r)}(X)$. Let $X = USV^\top$ be the singular value decomposition with orthogonal matrices $U\in \R^{d\times r}, V \in \R^{r\times r}$ and diagonal matrix $S \in \R^{r\times r}$, then $\bar{X} = UV^\top$. Building on these notations, we demonstrate that problem \eqref{prob:feasi} satisfies the restricted secant inequality (RSI) \cite{zhang2013gradient}, which serves as an alternative to the strong convexity in the linear convergence analysis of gradient-type methods. 
\begin{lem}[RSI] \label{lem:rsi}
For any $X \in \R^{d\times r}$ with $\| X - \bar{X} \| \leq \frac{1}{8}$, we have
\begin{equation*}
\label{eq:rsi} 
\iprod{\nabla \varphi(X)}{X - \bar{X}} \geq \| X - \bar{X} \|^2. 
\end{equation*}
\end{lem}

With the given RSI, applying the gradient descent (GD) update to \eqref{prob:feasi}, i.e.,
\begin{equation}
\label{eq:grad-feasi} 
X^{\rm pen}_{k+1} = X^{\rm pen}_k - \mu \nabla \varphi(X^{\rm pen}_k),
\end{equation}
yields the following linear convergence result.

\begin{lem}[Linear convergence of GD for \eqref{prob:feasi}]  \label{lem:linear-feasi}
Let the sequence $\{X_k^{\rm pen}\}$ be generated by \eqref{eq:grad-feasi} with $\mu = \frac{1}{3}$. Suppose that $\| X^{\rm pen}_1 - \overline{X^{\rm pen}_1} \| \leq \frac{1}{8}$, then we have
\begin{equation*}
\label{eq:grad-feasi-linear} 
\|X_{k+1}^{\rm pen} - \overline{X_{k+1}^{\rm pen}}\|^2 \leq \frac{2}{3}\|X_k - \overline{X_k^{\rm pen}} \|^2. 
\end{equation*}
\end{lem}

Note that the linear convergence of gradient descent follows from the established RSI. It would also be interesting to investigate Newton--Schulz-type updates for nearly orthogonal matrices and to exploit their superlinear local convergence \cite{higham2008functions,lakic1998computation} in the design and analysis of retraction-free algorithms for problem~\eqref{prob}.
The proofs of Lemmas \ref{lem:rsi} and \ref{lem:linear-feasi} are provided in Appendix~\ref{app:basic-proofs}.

\subsection{Landing on the Stiefel Manifold}

Building on the established linear convergence of gradient descent for problem \eqref{prob:feasi}, we are now able to show that the iterates generated by \eqref{eq:grad-it} will land on the Stiefel manifold eventually, and the limiting point is a stationary point of \eqref{prob}, i.e., $X_{\infty} \in {\rm St}(d,r)$, and $\grad f(X_\infty) = 0$. 

\begin{assum} \label{assum}
Suppose the following smoothness and stochasticity conditions hold:
\begin{itemize}[left=0pt]

\item For each $i$, the component function $f_i$ is continuously differentiable, and its Euclidean gradient $\nabla f_i$ is $L_f$-Lipschitz continuous over the convex hull of $\bar U_{\mathrm{St} (d, r)}\bigl (\tfrac18\bigr)$.

\item The stochastic gradient $g_k$ is an unbiased estimator of $\nabla f (X_k)$ with uniformly bounded variance, namely, for all $k=1, 2, \dots$,
$$
    \mathbb{E}[\, g_k\, ] \; =\; \nabla f (X_k),
    \qquad
    \mathbb{E}\bigl[\|\, g_k - \nabla f (X_k)\|^2\bigr]\; \le\; \sigma^2,
$$
where $\sigma > 0$ is a constant. 
\end{itemize}  
\end{assum}

For simplicity, define $\hat{\nabla} f_i(X) = \Pcal_{T_X {\rm St}(d,r)}( \nabla f_i(X))$. Note that $\hat{\nabla} f_i(X) = \grad f_i(X)$ whenever $X \in {\rm St}(d,r)$. 
We first have the following quadratic upper bound on $f_i$ and Lipschitz continuity of $\grad f_i$. 
\begin{lem}[Quadratic upper bound] \label{lem:lip}
Suppose that Assumption \ref{assum} holds. There exists a constant $L > 0$ such that for any $X, Y\in {\rm St}(d,r)$, and any $i$, the following quadratic upper bound holds:
\begin{equation}
\label{eq:qub} 
f_i(Y) \leq f_i(X) + \iprod{\grad f_i(X)}{Y-X} + \frac{L}{2} \|Y-X\|^2.
\end{equation}
In addition, there exists a constant $\hat{L} > 0$ such that for any $X \in {\rm St}(d,r), Y \in \bar{U}_{{\rm St}(d,r)}(\frac{1}{8})$, and any $i$,
\begin{equation}
\label{eq:grad-lip} 
\| \grad f_i(X) - \hat{\nabla} f_i(Y) \| \leq \hat{L} \|X - Y\|.
\end{equation}
\end{lem}

\begin{figure*}
\begin{subfigure}[t]{\threepanelwidth}
    \centering
    \includegraphics[width=\linewidth, ]{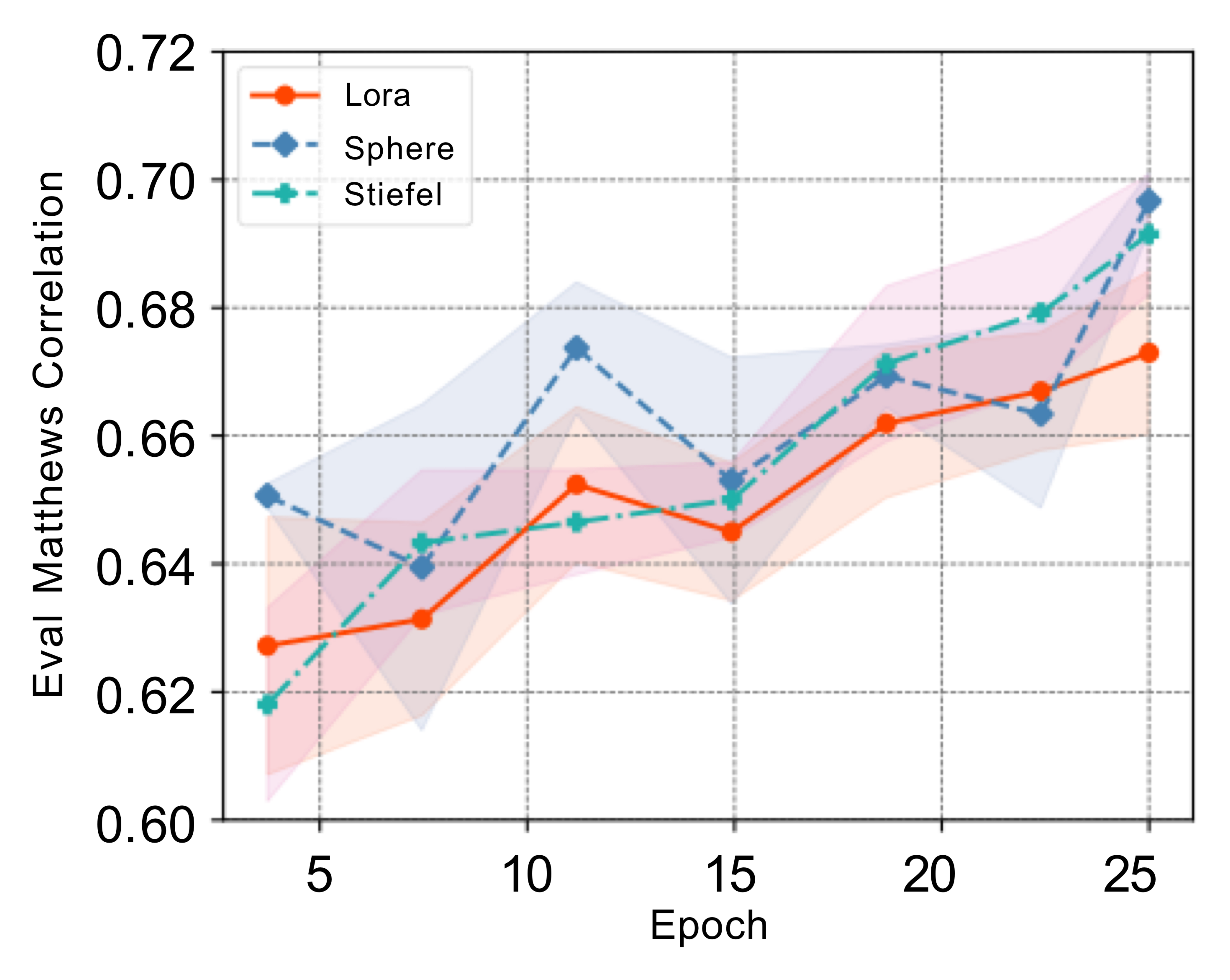}
    \caption{CoLA evaluation matthews correlation}
    \label{fig:sub-4}
\end{subfigure}
\hfill
\begin{subfigure}[t]{\threepanelwidth}
    \centering
    \includegraphics[width=\linewidth]{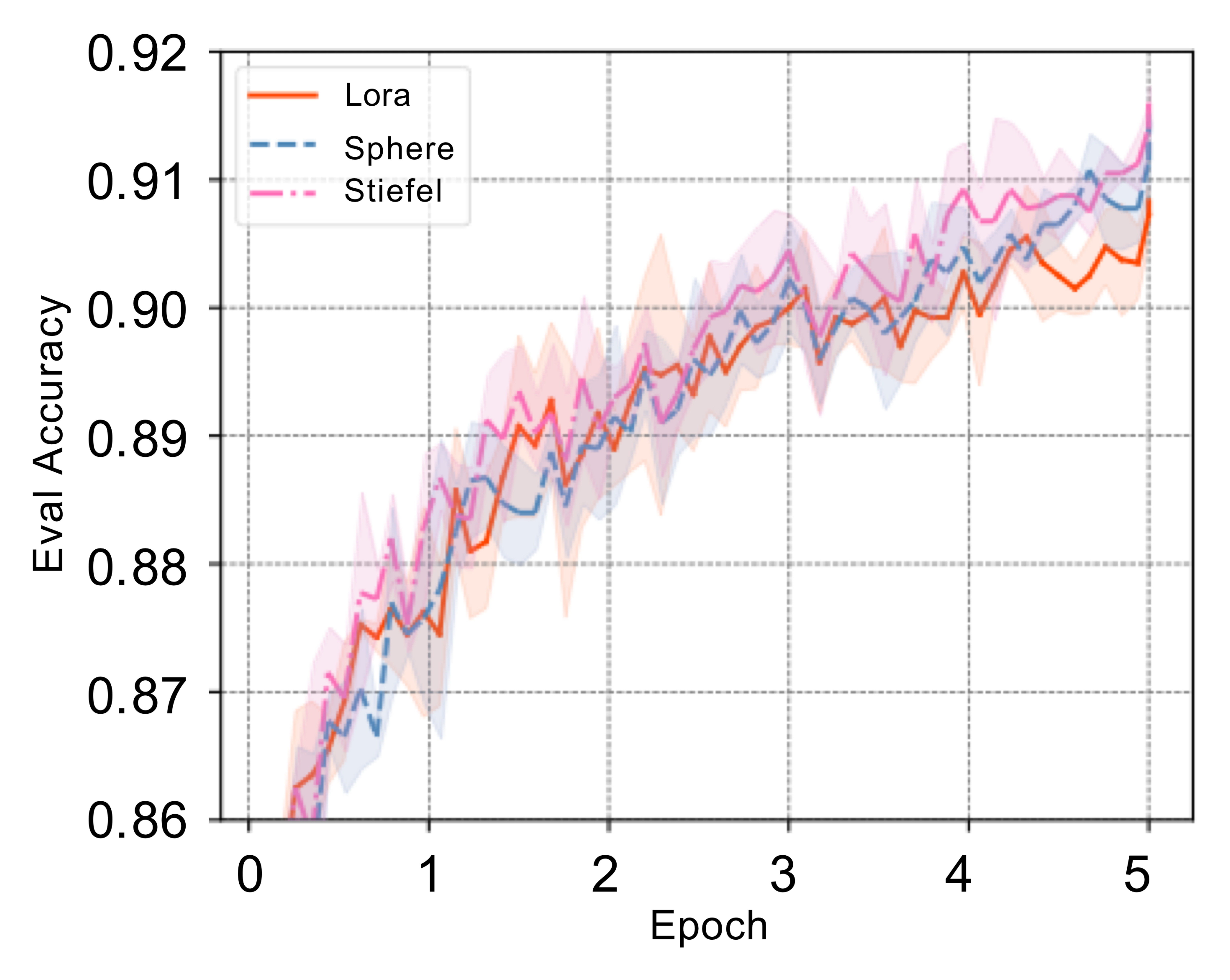}
    \caption{QQP evaluation accuracy}
    \label{fig:sub-5}
\end{subfigure}
\hfill
\begin{subfigure}[t]{\threepanelwidth}
    \centering
    \includegraphics[width=\linewidth]{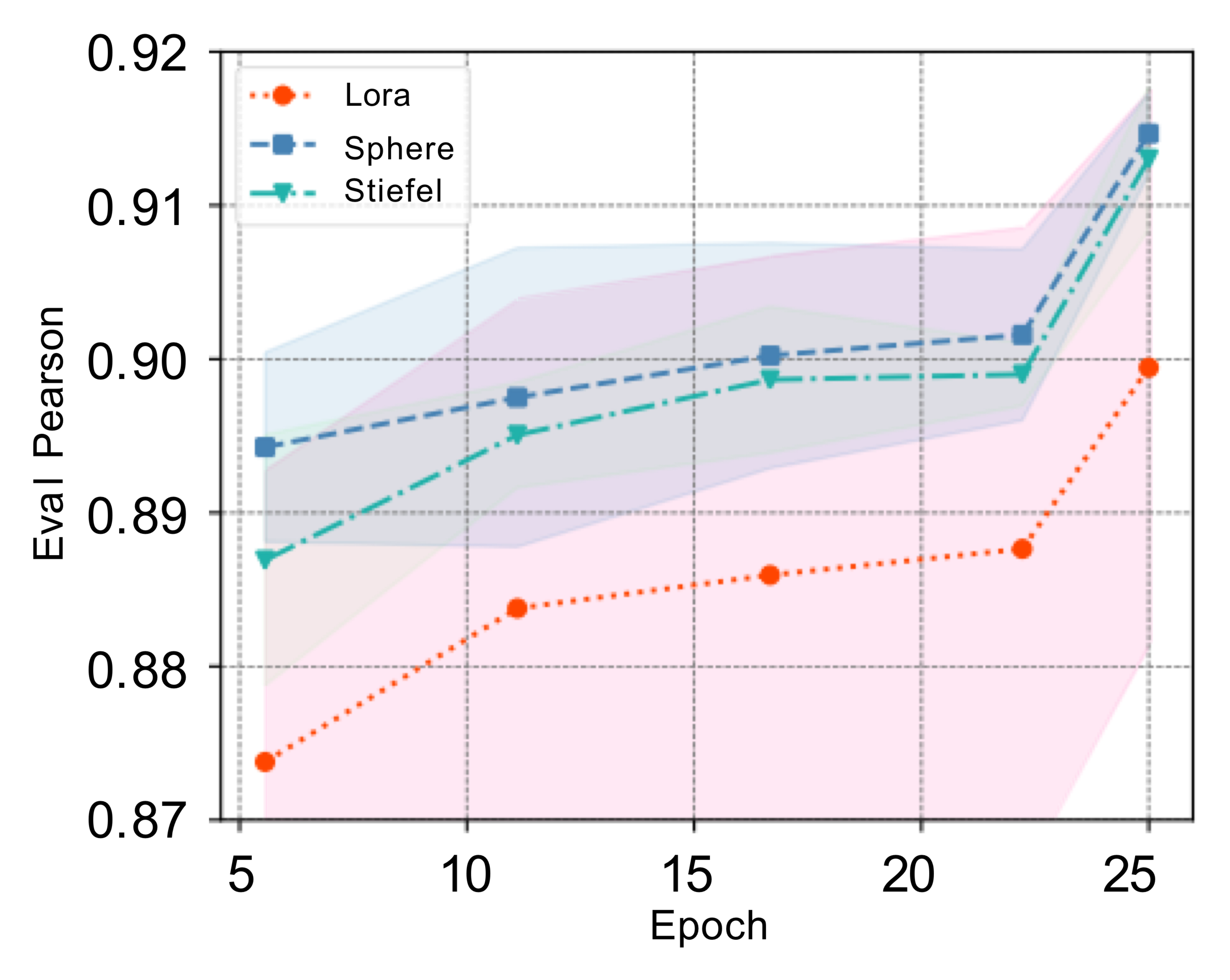}
    \caption{STS-B evaluation pearson}
    \label{fig:sub-6}
\end{subfigure}
\caption{Performance on the validation sets across three datasets. The COLA dataset is evaluated using the matthews
correlation metric, QQP is measured by accuracy, and STS-B is evaluated by Pearson correlation, all plotted against the number of epochs. }
\label{appen-exp-fig}
\end{figure*}

By the linear convergence result in Lemma \ref{lem:linear-feasi}, we have the following decay on the feasibility error.
\begin{lem}[Error bound of feasibility] \label{lem:err-feasi}
Let $\{X_k\}$ be the sequence generated by \eqref{eq:grad-it} with $\mu = \frac{1}{3}$. If $\| X_k - \bar{X}_k \| \leq \frac{1}{8}$, then
    \begin{equation}
    \label{eq:err-feasi} 
    \|X_{k+1} - \bar{X}_{k+1}\| \leq \sqrt{\tfrac{2}{3}} \| X_k - \bar{X}_k \| + \frac{5}{2}\alpha_k \| g_k \|. 
    \end{equation}
\end{lem}

The following one-step descent lemma on $f$ is crucial in establishing the convergence. Detailed proofs of Lemmas \ref{lem:lip}, \ref{lem:err-feasi}, and \ref{lem:descent} are provided in Appendix~\ref{app:basic-proofs}.

\begin{lem}[One-step descent for $f$] \label{lem:descent}
    Suppose that Assumption \ref{assum} holds. Let $\{X_k\}$ be the sequence generated by Algorithm \ref{alg:retr-free} with $\mu = \frac{1}{3}$. If $\| X_k - \bar{X}_k \| \leq \frac{1}{8}$, we have
\begin{equation} \label{eq:descent} 
\begin{aligned}
& \mathbb{E}_k[f(\bar{X}_{k+1})] - f(\bar{X}_k)  \\
\leq &  - (\alpha_k - (4\hat{L}^2 + 4L + 1) \alpha_k^2) \|\hat{\nabla} f(X_k)\|^2  + \\ 
&\frac{1}{2}\left(4\hat{D}_f + 9\hat{L}^2 + 8 L + 3 \right) \|X_k - \bar{X}_k\|^2 +\\
&\frac{1}{2}\|X_{k+1} - \bar{X}_{k+1}\|^2 +   7 (4\hat{L}^2 + 4L + 1) \alpha_k^2 \sigma^2,
\end{aligned}
\end{equation}
    where $\hat{D}_f:= \max_{X \in \bar{U}_{{\rm St}(d,r)}(\frac{1}{8})} \|\nabla f(X)\|$ and we use $\mathbb{E}_k[\cdot]$ to denote an expected value taken with respect to the distribution of the random variable in the estimation of $g_k$ given $X^k$.
\end{lem}

From the above lemma, the one-step decrease on $f$ is related to both the gradient norm of $f$ and the feasibility error. Regarding convergence, we need both $\grad f(X_k)$ and $\|X_k^\top X_k - I\|$ converge to 0. The following theorem shows that the retraction-free and penalty parameter-free update \eqref{eq:grad-it} converges.

\begin{thm}[Convergence of Algorithm \ref{alg:retr-free}] \label{thm}
Suppose that Assumption \ref{assum} holds. Let $\{X_k\}$ be the sequence generated by Algorithm \ref{alg:retr-free} with $\mu = \frac{1}{3}$ and $\| X_1 - \bar{X}_1 \| \; \le\; \tfrac18.$ Denote the total expectation by $\mathbb{E}[\cdot]  : = \mathbb{E}_{1}\, \mathbb{E}_2\cdots\mathbb{E}_k[\cdot].$ Then the following statements hold:

\begin{itemize}[left=0pt]
    \item {(Deterministic case)} If the full gradient is used (i.e.\ $\sigma=0$ in Assumption \ref{assum}) and the step size is kept constant as $\alpha_k \equiv \alpha \in \bigl (0, \tfrac1{2c_1}\bigr]$, for some sufficiently large $c_1>0$, then
\begin{equation*}
\label{eq:complexity-deterministic} 
\min_{k \leq K} \;\; \mathbb{E} \left[\|\hat{\nabla} f(X_k)\|^2 + \|X_k^\top X_k - I\|^2\right] = \mathcal{O}\left(\frac{1}{K}\right).
\end{equation*}
    \item {(Stochastic case)} If the step size decays as $ \alpha_k = \frac{\alpha_0}{\sqrt{k}},$ with $\alpha_0 \in \bigl (0, \tfrac1{2c_1}\bigr]$, for some sufficiently large $c_1>0$, then
\begin{equation*}
\label{eq:complexity} 
\min_{k \leq K} \;\; \mathbb{E}\left[\|\hat{\nabla} f(X_k)\|^2 + \|X_k^\top X_k - I\|^2\right]  = \mathcal{O}\left(\frac{\log K}{\sqrt{K}}\right)  
\end{equation*}
and 
\begin{equation*}
    \min_{k\leq K} \|X_k^\top X_k - I\|^2 = \mathcal{O}(\frac{\log K}{K}). 
\end{equation*}
\end{itemize}
\end{thm}

The proof of Theorem \ref{thm} is presented in Appendix~\ref{app:basic-proofs}.

\begin{rmk}
Compared with the landing algorithm \cite{ablin2022fast}, which targets only the squared Stiefel manifold and requires tuning both parameters $\alpha$ and $\mu$, our approach handles general Stiefel manifolds and necessitates the tuning of only $\alpha$, as established in Theorem \ref{thm}. In addition, the landing algorithm in \cite{ablin2022fast} converges only to a neighborhood whose size depends on the step size, as discussed in the paragraph following Proposition 10 of their paper. Moreover, our iteration complexity of $\mathcal{O}(1/K)$ is on par with retraction-based algorithms \cite{boumal2019global}.
\end{rmk}

\begin{rmk}
Our penalty-parameter-free analysis establishes an improved convergence rate of $\mathcal{O}\left(\frac{\log K}{K}\right)$ for constraint violation in the stochastic setting with decaying step sizes, improving upon the $\mathcal{O}\left(\frac{\log K}{\sqrt{K}}\right)$ rate reported in \cite{ablin2023infeasible}.
This improvement stems from our two-scale step-size scheme---decaying for the loss-gradient step while remaining constant for the penalty-gradient step---whereas \cite{ablin2023infeasible} employs a single-scale decaying step size for both steps, despite allowing an arbitrary but finite penalty parameter.
Furthermore, the proof of Theorem \ref{thm} relies on the linear decay of the constraint violation and the descent property of $f$, in contrast to the augmented Lagrangian-based framework of \cite{ablin2023infeasible}. We also remark that our penalty-parameter-free algorithm design and analysis are motivated by insights from multi-agent decentralized optimization \cite{nedic2009distributed,shi2015extra,qu2017harnessing,chen2021decentralized,deng2023decentralized}: the stepsize for the \emph{consensus (constraint violation)} term---which directly measures and penalizes the discrepancy across agents---is typically fixed to $1$, while only the stepsizes for the local loss-gradient updates require tuning. In this literature, the consensus error (i.e., the constraint violation) is also known to converge faster than the loss gradient in stochastic settings with decaying step sizes. Related two-time-scale ideas have also been explored in minimax optimization; see, e.g., \cite{lin2020gradient,lin2025two}.
\end{rmk}

Now, since Algorithm \ref{alg:manlora} with SGD-type update can be seen as a generalization of Algorithm \ref{alg:retr-free} to solve problems over the products of the Stiefel manifold or the Oblique manifold and the Euclidean space, we immediately have the following convergence result. The proof of Corollary \ref{coro} is given in Appendix~\ref{app:basic-proofs}.

\begin{cor}[Convergence of Algorithm \ref{alg:manlora}] \label{coro}
Suppose that the gradients $\nabla_A\mathcal{L} (BA)$ and $\nabla_B\mathcal{L} (BA)$ are Lipschitz continuous, and that $\hat g (A_k)$ and $\hat g (B_k)$ are unbiased, bounded-variance estimators of $\nabla_A\mathcal{L} (B_kA_k)$ and $\nabla_B\mathcal{L} (B_kA_k)$, respectively. Let $\{ (A_k, B_k)\}$ be the sequence generated by Algorithm \ref{alg:manlora} with $\mu = \frac{1}{3}$ and $\|B_0-\bar B_0\|\le1/8$, using an SGD-type update.
Then, if the step size is chosen as $\alpha_k = \frac{\alpha_0}{\sqrt{k}}$ with a small $\alpha_0 > 0$, we have
$$
    \begin{aligned}
    \min_{k \leq K} \; \; &\mathbb{E}\bigg[\| \nabla_A \mathcal{L} (B_kA_k)\|^2 + \|\hat{\nabla}_B \mathcal{L} (B_kA_k)\|^2 \\
    & \quad + \frac{\alpha_0}{\sqrt{k}}\|B_k^\top B_k - I\|^2\bigg] \leq \mathcal{O} \left(\frac{\log K}{\sqrt{K}}\right).
    \end{aligned}
$$
\end{cor}

The proof of the above corollary is based on Theorem \ref{thm} and the geometry of the product manifolds. The original Adam method may fail to converge~\cite{reddi2018convergence}; nevertheless, Euclidean Adam-type corrections can yield convergence guarantees for Algorithm \ref{alg:manlora} with Adam-type updates, e.g., AdaShift~\cite{zhou2019adashift}; see Appendix~\ref{app:adam-proof}.

\section{Experiments}
\label{section:exp}

In this section, we present comprehensive experimental results to evaluate the performance of Manifold-LoRA, i.e., Algorithm \ref{alg:manlora} with AdamW update,  across various tasks, including natural language understanding (NLU), question answering (QA), and natural language generation (NLG). We highlight the method's advantages in terms of convergence speed, downstream performance, and memory efficiency. All experiments follow a consistent setup to ensure fair comparison.

\subsection{Baselines and Implementation Details}
We compare Manifold-LoRA with several parameter-efficient fine-tuning (PEFT) baselines, including full fine-tuning, Adapter~\cite{houlsby2019parameter}, BitFit~\cite{zaken2021bitfit}, and LoRA~\cite{hu2021lora}. Variants of the Adapter method are omitted due to similar performance trends. Note that, although retraction or projection onto the oblique manifold is inexpensive, we still adopt the retraction-free update to validate our theoretical predictions, in particular the effectiveness of the landing property.

Our implementation is based on PyTorch~\cite{paszke2019pytorch}, Huggingface Transformers~\cite{wolf-etal-2020-transformers}, and OpenDelta~\cite{hu2023opendelta}. For a fair comparison, we ensure that all methods (including Adapter and LoRA) have approximately the same number of trainable parameters by aligning the bottleneck dimensions (e.g., 16 or 32). LoRA updates are scaled by a fixed hyperparameter $\alpha$ (typically 16 or 32, as in~\cite{hu2021lora}), and AdamW~\cite{loshchilov2017decoupled} is used as the optimizer with default exponential moving average parameters $\beta_1 = 0.9$, $\beta_2 = 0.999$. All experiments are conducted on NVIDIA A800 GPUs. \footnote{
The hyperparameters used for the GLUE benchmark, question-answering tasks,
and E2E benchmark are reported in
Tables~\ref{tab:glue_hyperparams}, \ref{tab:squad-hyper},
and \ref{tab:e2e-hyper}, respectively.
} To make a fair comparison, all hyperparameters such as batch size and learning rate scheduler, remain the same across experiments, except the additional parameters introduced by the Manifold-LoRA.

\subsection{Natural Language Understanding}
We evaluate Manifold-LoRA using the DeBERTaV3-base model~\cite{he2021debertav3} on the GLUE benchmark~\cite{wang2018glue}, which contains nine subdatasets including MNLI, SST-2, CoLA, QQP, QNLI, RTE, MRPC, and STS-B. For all GLUE tasks, we employ a consistent training configuration: a warmup ratio of 0.06, linear learning rate scheduling, a maximum sequence length of 256, weight decay set to 0.1, and a batch size of 32. The LoRA modules are applied to the query and value projection matrices ($W_q$ and $W_v$).

\paragraph{Performance Comparison}
Table~\ref{tab:glue} presents the GLUE results. Manifold-LoRA consistently outperforms LoRA and other baselines across most tasks. In particular, on the RTE and STS-B datasets, both sphere-constrained and Stiefel-constrained variants of Manifold-LoRA with rank $r=8$ outperform LoRA with $r=16$, indicating superior memory efficiency under equal memory budgets. Overall, the proposed algorithm achieves the best average performance across all methods, with the Oblique-constrained variant under rank $r=16$ attaining the highest average score of 88.63\%. Notably, even with a reduced rank of $r=8$, both the Stiefel- and Oblique-constrained Manifold-LoRA variants outperform the LoRA baseline with $r=16$, demonstrating superior parameter efficiency.

\paragraph{Convergence speed}
To further assess the convergence speed of our proposed algorithm, we compare the training loss trajectories throughout the optimization process. As shown in Figure~\ref{fig:glue-images}, Manifold-LoRA reaches the same training loss as the standard Adam optimizer in nearly half the number of epochs. In particular, on the CoLA dataset (Figure~\ref{fig:sub-1}), it exhibits almost 2× faster convergence.

We also track the evolution of validation metrics during training. As illustrated in Figure~\ref{appen-exp-fig}, Manifold-LoRA consistently outperforms vanilla LoRA. Notably, on the STS-B dataset (Figure~\ref{fig:sub-6}), our method achieves a significantly larger performance margin. For the CoLA and QQP datasets, Manifold-LoRA shows slight improvements over LoRA, demonstrating steady performance gains across tasks.

\paragraph{Stability across random seeds}
Results are averaged over five random seeds, with shaded areas indicating variance. As shown in Figure~\ref{appen-exp-fig}, Manifold-LoRA exhibits smaller variance compared to LoRA, confirming its stability. 

The focus of this work is to use manifold geometry to accelerate LoRA fine-tuning. A systematic study of how such geometric constraints affect generalization is beyond the scope of this paper and is left for future work.

\subsection{Question Answering}
We further fine-tune DeBERTaV3-base on SQuAD v1.1~\cite{rajpurkar2016squad} and SQuADv2.0~\cite{rajpurkar2018know} using Manifold-LoRA. The evaluation primarily focuses on F1 score and exact match accuracy to assess both the completeness and preciseness of predicted answers. For both SQuADv1.1 and SQuADv2.0, we adopt a consistent training setup: a warmup ratio of 0.06, a linear learning rate schedule, a weight decay of 0.1, and a batch size of 64. The learning rate is set to 3e-3, and training is conducted for 4 epochs. Following the LoRA framework, all low-rank modules are inserted into $W_q$, $W_k$, $W_v$, $W_o$, $FC_1$, and $FC_2$.

\paragraph{Performance comparison}
Table \ref{tab:squad}  summarizes experimental results when we fine-tune DeBERTaV3-base under different rank settings. For Adapter, we set the rank to 16 and 32, and for LoRA and our method, the rank is set to 8 and 16. It can be observed that Manifold-LoRA achieves superior performance with fewer trainable parameters. For instance, with the Stiefel constraint, Manifold-LoRA achieves an F1 score of 89.22 and an Exact Match score of 86.41 on the SQuADv2.0 dataset, significantly surpassing all other baselines.

\paragraph{Training efficiency and manifold consistency}
We also evaluate training loss, validation Exact Match, and validation F1 score across training epochs on the SQuADv2.0 dataset, as shown in Figure~\ref{fig:squadv2}. Across the three metrics, Manifold-LoRA exhibits a significantly faster convergence speed, nearly twice that of LoRA, demonstrating a faster loss reduction and a faster improvement in evaluation performance. Additionally, we observe that the models with Stiefel and Oblique constraints follow similar trends in both the loss and evaluation metrics, suggesting that both geometric constraints can achieve comparable effectiveness on question answering tasks.

To further verify whether the learned low-rank matrix $B$ effectively lies on the target manifold, we visualize its structure by plotting the heatmaps of $B^\top B$ in Figure~\ref{fig:squadv2-heatmap}. Specifically, we extract the checkpoint from the DeBERTa-base model after the second training epoch on the SQuADv2.0 dataset, and select several representative layers from the second and third transformer blocks for analysis. The resulting heatmaps reveal that $B^\top B$ closely adheres to the expected manifold geometry,  validating the effectiveness of our manifold-constrained design. 

\subsection{Natural Language Generation}
Having demonstrated superior performance on natural language understanding and question answering tasks, we further extend our experiments to evaluate the effectiveness of our proposed method on a natural language generation task. Specifically, we conduct experiments on the E2E NLG Challenge dataset~\cite{novikova2017e2e}, using GPT-2 Medium and GPT-2 Large as backbone models. The E2E dataset comprises approximately 50K examples with 8 distinct semantic fields. It provides multiple reference outputs for each input table, with an average output length of 22.9 tokens. For both GPT-2 Medium and Large models on the E2E benchmark, we use a linear learning rate schedule with 500 warmup steps, a weight decay of 0.01, and no LoRA dropout. The models are trained for 5 epochs with a batch size of 8 and a learning rate of 2e-4.
\paragraph{Performance comparison} 
To ensure a fair comparison, we follow the same hyperparameter settings as in the original LoRA paper, except for the additional parameters introduced by Manifold-LoRA. The numerical results are summarized in Table~\ref{tab:e2e}. It is evident that Manifold-LoRA achieves superior performance across all five metrics with limited trainable parameters.

\subsection{Scaling Experiments}
\paragraph{Experimental Settings}
We largely follow the experimental protocol of prior work on LoRA-based fine-tuning with manifold-constrained optimization to ensure a fair and consistent comparison \cite{park2025riemannian}. In particular, we adopt the same model architectures (LLaMA3.2-1B/3B and LLaMA3-8B \cite{dubey2024llama}), downstream benchmarks (SQuAD, QuAC, GSM8K \cite{cobbe2021training}, and MATH), and evaluation metrics as in the reference study. For all experiments, LoRA is applied to the same set of layers with identical rank and scaling configurations, and models are fine-tuned under the same training budget. 

\paragraph{Reading Comprehension}
Table \ref{tab:rc_results} reports results on SQuAD and QuAC when scaling our method from 1B to 8B models, where the method ``Stiefel'' is from \cite{park2025riemannian}. Across all model sizes, Manifold-LoRA consistently outperforms other baselines on both datasets. Notably, the performance advantage on SQuAD and QuAC is maintained as the model size increases, indicating that our optimization approach scales well for extractive reading comprehension. On QuAC, which emphasizes semantic understanding and contextual coherence across conversational turns, the oblique manifold-constrained optimization demonstrates more consistent gains than Stiefel manifold-based alternatives.

\paragraph{Mathematical Reasoning}

Results on GSM8K and MATH are summarized in Table \ref{tab:math_results}. Stiefel-based Manifold-LoRA yields consistent improvements over AdamW across all model scales, with particularly strong gains on the more challenging MATH benchmark. The improvements become more pronounced for larger models, especially LLaMA3-8B, indicating that Stiefel manifold-constrained optimization is effective at preserving and enhancing complex mathematical reasoning capabilities as model capacity increases.
\begin{table}[!t]
\centering
\scriptsize
\setlength{\tabcolsep}{3.2pt}
\caption{Reading comprehension results on SQuAD and QuAC (F1/EM).}
\label{tab:rc_results}
\begin{tabular*}{\columnwidth}{@{\extracolsep{\fill}}l l l c c@{}}
\toprule
\textbf{Model} & \textbf{Method} & \textbf{Opt.} & \textbf{SQuAD} & \textbf{QuAC} \\
\midrule
\multirow{4}{*}{LLaMA3.2-1B} & \multirow{4}{*}{LoRA} & Stiefel & 67.9/55.7 & 50.4 \\
            &      & AdamW   & 64.1/51.5 & 45.9 \\
            &      & \textbf{Manifold-LoRA (Stiefel)} & 69.4/56.9 & 51.8 \\
            &      & \textbf{Manifold-LoRA (Sphere)} & \textbf{70.6/57.4} & \textbf{52.0} \\
\midrule
\multirow{4}{*}{LLaMA3.2-3B} & \multirow{4}{*}{LoRA} & Stiefel & 80.3/72.1 & 61.8 \\
            &      & AdamW   & 78.6/67.4 & 57.5 \\
            &      & \textbf{Manifold-LoRA (Stiefel)} & \textbf{81.7/72.8} & \textbf{62.3} \\
            &      & \textbf{Manifold-LoRA (Sphere)} & 80.0/71.9 & 61.5 \\
\midrule
\multirow{4}{*}{LLaMA3-8B}   & \multirow{4}{*}{LoRA} & Stiefel & 88.1/79.7 & 69.7 \\
            &      & AdamW   & 84.3/74.6 & 65.8 \\
            &      & \textbf{Manifold-LoRA (Stiefel)} & 89.8/81.2 & 70.8 \\
            &      & \textbf{Manifold-LoRA (Sphere)} & \textbf{90.4/82.0} & \textbf{71.3} \\
\bottomrule
\end{tabular*}
\end{table}

\begin{table}[!t]
\centering
\scriptsize
\setlength{\tabcolsep}{3.2pt}
\caption{Mathematical reasoning accuracy on GSM8K and MATH.}
\label{tab:math_results}
\begin{tabular*}{\columnwidth}{@{\extracolsep{\fill}}l l l c c@{}}
\toprule
\textbf{Model} & \textbf{Method} & \textbf{Opt.} & \textbf{GSM8K} & \textbf{MATH} \\
\midrule
\multirow{4}{*}{LLaMA3.2-1B} & \multirow{4}{*}{LoRA} & Stiefel & 35.4 & 26.5 \\
            &      & AdamW   & 20.5 & 21.4 \\
            &      & \textbf{Manifold-LoRA (Stiefel)} & \textbf{37.2} & \textbf{27.9} \\
            &      & \textbf{Manifold-LoRA (Sphere)} & 34.9 & 26.3 \\
\midrule
\multirow{4}{*}{LLaMA3.2-3B} & \multirow{4}{*}{LoRA} & Stiefel & 43.4 & 33.5 \\
            &      & AdamW   & 29.1 & 27.7 \\
            &      & \textbf{Manifold-LoRA (Stiefel)} & \textbf{46.5} & \textbf{35.7} \\
            &      & \textbf{Manifold-LoRA (Sphere)} & 44.5 & 34.8 \\
\midrule
\multirow{4}{*}{LLaMA3-8B}    & \multirow{4}{*}{LoRA} & Stiefel & 58.8 & 22.5 \\
            &      & AdamW   & 54.7 & 19.3 \\
            &      & \textbf{Manifold-LoRA (Stiefel)} & \textbf{60.5} & \textbf{23.8} \\
            &      & \textbf{Manifold-LoRA (Sphere)} & 57.6 & 21.9 \\
\bottomrule
\end{tabular*}
\end{table}

\begin{table}[!t]
  \scriptsize
  \centering
  \setlength{\tabcolsep}{3pt}
  \caption{Manifold-specific hyperparameters of Manifold-LoRA for the E2E benchmark.}
  \label{tab:e2e-hyper}
  \begin{tabular*}{\columnwidth}{@{\extracolsep{\fill}}l l l l@{}}
    \toprule
    Method & Hyperparameter & GPT-2(M) & GPT-2(L)  \\
    \midrule
    Sphere  & $\mu$   & 1   & 0.9 \\
    ($r=4$) & Lower   & 0.5 & 0.5 \\
            & Upper   & 2   & 2 \\
    \midrule
    Stiefel  & $\mu$   & 1   & 1.1 \\
    ($r=4$)  & Lower   & 0.5 & 0.5 \\
             & Upper   & 4   & 2 \\
    \bottomrule
  \end{tabular*}
\end{table}

\begin{table}[!t]
  \scriptsize
  \centering
  \setlength{\tabcolsep}{3pt}
  \caption{Manifold-specific hyperparameters of Manifold-LoRA for question answering tasks.}
  \label{tab:squad-hyper}
  \begin{tabular*}{\columnwidth}{@{\extracolsep{\fill}}l l l l@{}}
    \toprule
    Method & Hyperparameter & SQuADv1.1 & SQuADv2.0  \\
    \midrule
    Sphere ($r=8$)  & $\mu$   & 0.85 & 0.85 \\
                    & Lower   & 0.25 & 0.25 \\
                    & Upper   & 0.75 & 0.5 \\
    \midrule
    Sphere ($r=16$) & $\mu$   & 0.9 & 0.85 \\
                    & Lower   & 0.25 & 0.25 \\
                    & Upper   & 0.5 & 0.5 \\
    \midrule
    Stiefel ($r=8$) & $\mu$   & 0.85 & 0.85 \\
                    & Lower   & 0.25 & 0.25 \\
                    & Upper   & 0.5 & 0.5 \\
    \midrule
    Stiefel ($r=16$)& $\mu$   & 0.9 & 0.85 \\
                    & Lower   & 0.25 & 0.25 \\
                    & Upper   & 0.5 & 0.5 \\
    \bottomrule
  \end{tabular*}
\end{table}

\begin{table*}
\centering
\footnotesize
\caption{Hyperparameter configurations of Manifold-LoRA for GLUE benchmark.}
\label{tab:glue_hyperparams}
\resizebox{\widecompactwidth}{!}{
\begin{tabular}{llcccccccc}
\toprule
Method & Hyperparameter & MNLI & SST-2 & CoLA & QQP & QNLI & RTE & MRPC & STS-B \\
\midrule
\multirow{2}{*}{Baseline} & Epochs        & 7 & 24 & 25 & 5 & 5 & 50 & 30 & 25 \\
                          & Learning Rate & 5e-4 & 8e-4 & 5e-4 & 5e-4 & 1.2e-3 & 1.2e-3 & 1e-3 & 2.2e-3 \\
\midrule
Sphere ($r=16$) & $\mu$    & 1 & 0.9 & 0.8 & 0.9 & 0.95 & 1.2 & 0.85 & 0.9 \\
                & Lower    & 0.25 & 0.25 & 0.5 & 0.5 & 0.5 & 0.5 & 1 & 1 \\
                & Upper    & 2 & 2 & 2 & 4 & 2 & 2 & 4 & 4 \\
\midrule
Sphere ($r=8$)  & $\mu$    & 0.95 & 0.95 & 1 & 0.9 & 1 & 0.9 & 0.85 & 1 \\
                & Lower    & 2 & 0.5 & 1 & 0.5 & 0.5 & 0.25 & 2 & 1 \\
                & Upper    & 8 & 2 & 8 & 2 & 2 & 0.5 & 4 & 8 \\
\midrule
Stiefel ($r=16$)& $\mu$    & 0.8 & 0.85 & 0.95 & 0.9 & 0.95 & 1.2 & 0.8 & 1 \\
                & Lower    & 2 & 0.5 & 2 & 0.5 & 0.5 & 0.5 & 1 & 1 \\
                & Upper    & 8 & 1 & 8 & 4 & 1 & 2 & 4 & 16 \\
\midrule
Stiefel ($r=8$) & $\mu$    & 0.8 & 0.95 & 0.95 & 0.9 & 0.85 & 0.9 & 1 & 1 \\
                & Lower    & 2 & 0.5 & 2 & 0.5 & 0.5 & 0.25 & 1 & 1 \\
                & Upper    & 8 & 2 & 8 & 2 & 2 & 1 & 4 & 16 \\
\bottomrule
\end{tabular}
}
\end{table*}

\begin{table*}
  \centering
  \footnotesize
  \caption{We present results using DeBERTaV3-base on the GLUE benchmark. For MNLI, we report the accuracy (combining matched and mismatched sets), with the left panel representing matched subset and the right panel representing mismatched subset. For CoLA, we report Matthew’s correlation, and for STS-B, we report Pearson correlation. For all other tasks, we report accuracy. All metrics are same as the original LoRA paper \cite{hu2021lora}. Higher values are better for all metrics. The best results are highlighted in \textbf{bold}.}
  \label{tab:glue}
  \resizebox{\widecompactwidth}{!}{
  \begin{tabular}{
  l
  c
  S[table-format=2.2]@{/}S[table-format=2.2]
  S[table-format=2.2]
  S[table-format=2.2]
  S[table-format=2.2]@{/}S[table-format=2.2]
  S[table-format=2.2]
  S[table-format=2.2]
  S[table-format=2.2]
  S[table-format=2.2]
  S[table-format=2.2]
}
   \toprule
    {Method} & { \# Params} & \multicolumn{2}{c}{MNLI} & {SST-2} & {CoLA} & \multicolumn{2}{c}{QQP} & {QNLI} & {RTE} & {MRPC} & {STS-B} & {All} \\
    {} & {} & \multicolumn{2}{c}{Acc} & {Acc} & {Mcc} & {Acc} & {F1} & {Acc} & {Acc} & {Acc} & {Corr} & {Ave.} \\
    \midrule
    \small{Full FT} &  184.42M & 90.45 & \textbf{90.60} & 95.48 & 68.17 & \textbf{91.99} & \textbf{89.12} & 93.60 & 79.28 & 88.93 & 90.92 & 87.85 \\
    \small{Adapter} & 0.61M & 90.13 & 90.16 & 94.86 & 69.37 & 91.38 & 88.46 & 93.54  & 81.87 & 89.12 & 91.52 & 88.06\\
    \small{BitFit} & 0.06M & 87.08 & 86.39 & 94.88 & 69.11 & 87.96 & 84.35 & 92.19 & 76.52 & 87.06 & 90.96 & 85.65 \\
    \small{LoRA$_{r=8}$} & 0.30M & 90.20 & 90.08 & 94.93 & 68.14 & 90.78 & 87.68 & 93.85 & 80.15 & 90.40 & 90.29 & 87.60 \\
    \small{LoRA$_{r=16}$} & 0.59M & 90.44 & 90.12 & 95.41 & 68.19 & 90.92 & 87.77 & 94.00 & 80.58 & 90.20 & 90.34 & 87.74\\
    \small{Sphere$_{r=8}$} & 0.30M & 90.37 & 90.09 & 95.48 & 69.55 & 91.25 & 88.34 & 94.02 & 82.44 & 91.55 & 91.26 & 88.44 \\
    \small{Sphere$_{r=16}$} & 0.59M & \textbf{90.52} & 90.19 & 95.64 & \textbf{70.14} & 91.46 & 88.65 & \textbf{94.29} & 82.16 & \textbf{91.67} & \textbf{91.59} & \textbf{88.63} \\
    \small{Stiefel$_{r=8}$} & 0.30M & 90.25 & 89.99 & 95.46 & 69.85 & 91.44 & 88.60 & 94.09 & \textbf{83.16} & 91.18 & 91.22 & 88.52\\
    \small{Stiefel$_{r=16}$} & 0.59M & 90.26 & 90.28 & \textbf{95.76} & 68.92 & 91.71 & 89.00 & 94.10 & 82.16 & 91.10 & 91.51 & 88.48 \\
    \hline
\end{tabular}
}
\end{table*}

\begin{table*}[!t]
\centering
\footnotesize
\caption{Results with DeBERTaV3-base on SQuAD v1.1 and SQuADv2.0. We report F1 and Exact Match (EM). The best results in each setting are shown in \textbf{bold}.}
\resizebox{\smallresultwidth}{!}{
\begin{tabular}{
    l
    c
    c
    S[table-format=2.2]
    S[table-format=2.2]
    S[table-format=2.2]
    S[table-format=2.2]
}
\toprule
Methods & {\# Params} & {Rank} & \multicolumn{2}{c}{SQuADv1.1} & \multicolumn{2}{c}{SQuADv2.0} \\
\cmidrule(lr){4-5} \cmidrule(lr){6-7}
& & & {F1 Score} & {Exact Match} & {F1 Score} & {Exact Match} \\
\midrule
Full FT     & 184.42M & -- & 92.85 & 86.30 & 87.58 & 84.30 \\ 
Adapter     & 0.61M   & 16 & 93.41 & 87.46 & 88.23 & 85.30 \\
Adapter     & 1.22M   & 32 & 93.51 & 87.53 & 88.36 & 85.42 \\
Bitfit      & 0.07M   & -- & 88.79 & 80.26 & 87.19 & 74.21 \\
LoRA        & 1.33M   & 8  & 93.88 & 87.90 & 88.52 & 85.56 \\
LoRA        & 2.65M   & 16 & 93.75 & 87.94 & 88.81 & 85.90 \\
Sphere      & 1.33M   & 8  & \textbf{94.25} & 88.51 & 89.20 & 86.33 \\
Sphere      & 2.65M   & 16 & 94.03 & 88.32 & 89.03 & 86.15 \\
Stiefel     & 1.33M   & 8  & 94.23 & \textbf{88.68} & 89.09 & 86.35 \\
Stiefel     & 2.65M   & 16 & 94.04 & 88.25 & \textbf{89.22} & \textbf{86.41} \\
\bottomrule
\end{tabular}
}
\label{tab:squad}
\end{table*}

\begin{figure*}[!t]
\small
\centering
\begin{subfigure}[t]{\mediumfloatwidth}
    \centering
    \includegraphics[width=\linewidth]
    {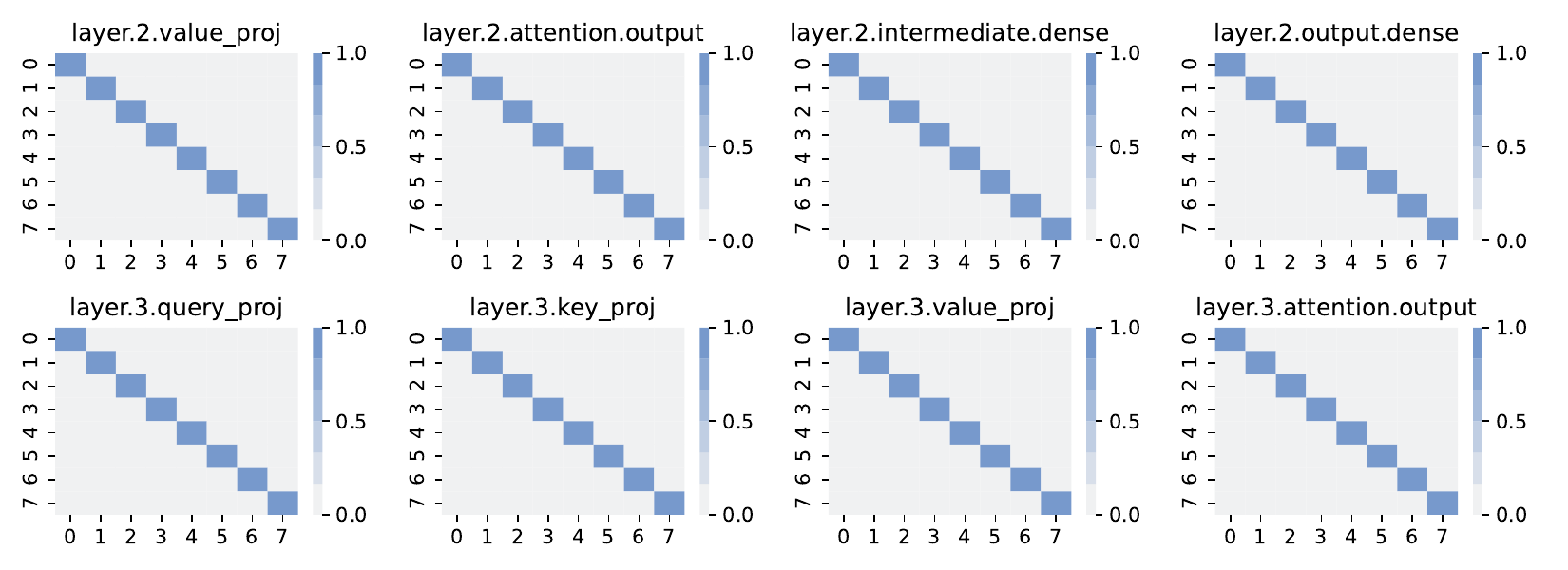}
    
    \label{fig:squad2-Stiefel}
\end{subfigure}

\begin{subfigure}[t]{\mediumfloatwidth}
    \centering
    \includegraphics[width=\linewidth]
    {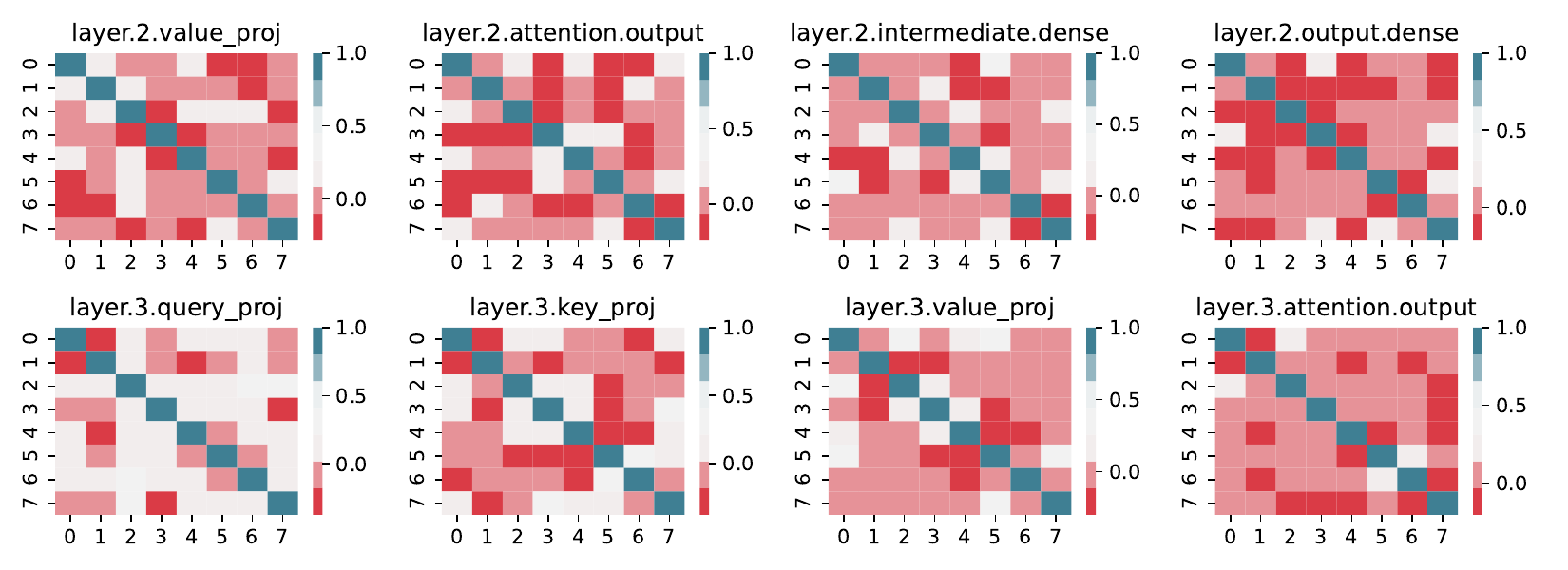}
    
    \label{fig:squad2-sphere}
\end{subfigure}
\caption{ The heat map of $B^\top B$ with the Stiefel manifold (the first and second rows) and the Oblique manifold (the third and fourth rows) at the end  of training on SQuADv2.0 dataset.}
\label{fig:squadv2-heatmap}
\end{figure*}

\begin{table*}[!t]
\footnotesize
\centering
\caption{GPT-2 medium (M) and large (L) models evaluated on the E2E NLG Challenge. * denotes results from previously published works.}
\label{tab:e2e}
\resizebox{\smallresultwidth}{!}{
\begin{tabular}{
    l
    c
    S[table-format=3.2]
    c
    S[table-format=3.2]
    S[table-format=3.2]
    c
}
\toprule
Model & Parameters & {BLEU} & {NIST} & {MET} & {ROUGE-L} & {CIDEr} \\
\midrule
GPT-2 M (FT)\textsuperscript{*} & 354.92M & 68.2 & 8.62 & 46.2 & 71.0 & 2.47 \\
GPT-2 M (Adapter\textsuperscript{L})\textsuperscript{*} & 11.09M & 68.9 & 8.71 & 46.1 & 71.3 & 2.47 \\
GPT-2 M (Adapter\textsuperscript{H})\textsuperscript{*} & 11.09M & 67.3 & 8.50 & 46.0 & 70.7 & 2.44 \\
GPT-2 M (FT\textsuperscript{Top2})\textsuperscript{*} & 25.19M & 68.1 & 8.59 & 46.0 & 70.8 & 2.41 \\
GPT-2 M (PreLayer)\textsuperscript{*} & 0.35M & 69.7 & 8.81 & 46.1 & 71.4 & 2.49 \\
GPT-2 M (LoRA) & 0.35M & 68.9 & 8.69 & 46.5 & 71.5 & 2.51 \\
GPT-2 M (Stiefel) & 0.35M & 70.1 & 8.82 & \textbf{46.8} & \textbf{71.7} & \textbf{2.53} \\
GPT-2 M (Sphere) & 0.35M & \textbf{70.3} & \textbf{8.83} & 46.7 & \textbf{71.7} & 2.52 \\
\midrule
GPT-2 L (FT)\textsuperscript{*} & 774.03M & 68.5 & 8.78 & 46.0 & 69.9 & 2.45 \\
GPT-2 L (Adapter\textsuperscript{L})\textsuperscript{*} & 23.00M & 68.9 & 8.70 & 46.1 & 71.3 & 2.45 \\
GPT-2 L (PreLayer)\textsuperscript{*} & 0.77M & 70.3 & 8.85 & 46.2 & 71.7 & 2.47 \\
GPT-2 L (LoRA) & 0.77M & 70.1 & 8.82 & 46.7 & 72.0 & 2.53 \\
GPT-2 L (Stiefel) & 0.77M & 70.4 & 8.86 & \textbf{46.8} & 72.1 & 2.53 \\
GPT-2 L (Sphere) & 0.77M & \textbf{70.9} & \textbf{8.92} & \textbf{46.8} & \textbf{72.5} & \textbf{2.55} \\
\bottomrule
\end{tabular}
}
\end{table*}

\section{Conclusion}
\label{section:conlusion}
Optimization over the Stiefel manifold has been widely used in machine learning tasks. 
In this work, we develop a retraction-free and penalty parameter-free gradient method, and prove that the generated iterates eventually land on the manifold and achieve the optimality simultaneously. Moreover, our convergence theory enables the use of a constant step size, improving on previous results that only ensured convergence to a neighborhood. We then apply this landing theory to avoid the possible redundancy of LoRA fine-tuning in LLMs. Specifically, we reformulate the LoRA fine-tuning as an optimization problem over the product of the Stiefel or Oblique manifold and Euclidean space, and propose a new algorithm, Manifold-LoRA, which incorporates a careful analysis of step sizes to enable fast training using the landing properties. Extensive experimental results demonstrate that our approach not only accelerates the training process but also yields significant performance improvements.

Our study suggests several potential directions for future research. Although the established landing theory focuses on the Stiefel manifold, e.g., the generalized Stiefel manifold \cite{vary2024optimization}, extending this theory to general manifolds, 
is one potential direction. 
Additionally, evaluating the performance of Manifold-LoRA on LLMs with billions of parameters would be valuable. A systematic understanding of the generalization behavior induced by manifold-constrained adaptation is also an interesting direction for future research. Due to the heterogeneity of different layers, incorporating adaptive ranks for $\Delta W$ across different layers is another possible direction. 
This may be achievable by adding sparsity regularization to the coordinate matrix $A$. 

\section*{Acknowledgement} We thank the Associate Editor and the two reviewers for their constructive comments and suggestions, which have substantially improved the manuscript.

{\footnotesize
\bibliographystyle{IEEEtran}
\bibliography{ref}
}

\clearpage
\onecolumn
\appendices
\numberwithin{equation}{section}
\numberwithin{table}{section}
\numberwithin{thm}{section}
\numberwithin{lem}{section}
\numberwithin{cor}{section}
\numberwithin{prop}{section}
\numberwithin{assum}{section}
\numberwithin{cond}{section}
\numberwithin{rmk}{section}

\section{Proofs for the Retraction-Free Convergence Analysis}
\label{app:basic-proofs}
This appendix provides the detailed arguments supporting the convergence results in Section~\ref{sec:con}. Throughout, $\bar X$ denotes a closest point to $X$ on the Stiefel manifold.

\subsection{Proof of Lemma \ref{lem:rsi}}

\begin{proof}%[Proof of Lemma \ref{lem:rsi}]
Let the singular-value decomposition of $X$ be $ X = U S V^\top,$ and write $\bar X = U V^\top$.  Then the distance from $X$ to the Stiefel manifold is
$$
    \mathrm{dist}(X,\mathrm{St}(d,r)) \;=\;\|X-\bar X\|
    \;=\;\|s - \mathbf1\|_2,
$$
where $s = \operatorname{diag}(S)$ is the vector of singular values of $X$. Under the assumption $\|X-\bar X\|\le\frac18$, each singular value satisfies
$$
    \frac78 \;\le\; s_i \;\le\; \frac98
    \quad\text{for all }i.
$$
We now compute the inner product between the gradient of
$\varphi(X)=\tfrac14\|X^\top X - I\|_F^2$ and the deviation $X-\bar X$.  First observe
$$
    \nabla \varphi(X)
    = X\bigl(X^\top X - I\bigr),
$$
so
$$
    \bigl\langle\nabla\varphi(X),\,X-\bar X\bigr\rangle
    = \bigl\langle X(X^\top X - I),\,X-\bar X\bigr\rangle.
$$
Substitute $X=USV^\top$ and $\bar X=UV^\top$, then
$$
    \begin{aligned}
    \bigl\langle\nabla\varphi(X),\,X-\bar X\bigr\rangle
    &= \bigl\langle U S V^\top\bigl(V S^2 V^\top - I\bigr),\,U S V^\top - U V^\top\bigr\rangle\\
    &= \bigl\langle U\,(S^3 - S)\,V^\top,\;U\,(S - I)\,V^\top\bigr\rangle
    = \mathrm{tr}\bigl((S^3 - S)(S - I)\bigr).
    \end{aligned}
$$
Since $(S^3 - S)(S - I)$ is diagonal, its trace is simply
$$
    \sum_i \bigl(s_i^3 - s_i\bigr)(s_i - 1)
    = \sum_i s_i(s_i+1)(s_i-1)^2.
$$
Because each $s_i\ge7/8$, one checks that
$$
    s_i(s_i+1)\;\ge\;\frac78\cdot\frac{15}{8}
    =\frac{105}{64}>\frac32,
$$
and hence 
$$
    \bigl\langle\nabla\varphi (X), \, X-\bar X\bigr\rangle
    = \sum_i s_i (s_i+1) (s_i-1)^2
    \; \ge\; \frac32\sum_i (s_i-1)^2
    =\frac32\|s-\mathbf1\|_2^2
    =\frac32\|X-\bar X\|^2.
$$
This completes the proof.
\end{proof}

\subsection{Proof of Lemma \ref{lem:linear-feasi}}

\begin{proof}%[Proof of Lemma \ref{lem:linear-feasi}]
Let the singular-value decomposition of the penalized iterate $X_k^{\rm pen}$ be $X_k^{\rm pen} = U_k \, S_k\, V_k^\top,$ and write $\overline{ X_k^{\rm pen}} = U_k V_k^\top$ for its closest point on the Stiefel manifold. Since
$$
    \nabla\varphi (X_k^{\rm pen})
    = X_k^{\rm pen}\bigl ( (X_k^{\rm pen})^\top X_k^{\rm pen} - I\bigr)
    = U_k\, (S_k^3 - S_k)\, V_k^\top,
$$
its squared Frobenius norm is
$$
    \|\nabla\varphi (X_k^{\rm pen})\|_F^2
    = \mathrm{tr}\bigl ( (S_k^3 - S_k)^2\bigr)
    = \sum_i\bigl (s_i^3 - s_i\bigr)^2
    = \sum_i s_i^2\, (s_i+1)^2\, (s_i-1)^2.
$$
Under the same small-deviation assumption $\|X_k^{\rm pen}-\overline{ X_k^{\rm pen}}\|\le\frac18$, each singular value $s_i$ lies in $[\tfrac78, \tfrac98]$, so
$$
    s_i^2\, (s_i+1)^2
    \; \le\; \Bigl (\tfrac98\Bigr)^2\Bigl (\tfrac98+1\Bigr)^2
    \; <\; 6.
$$
Hence
$$
    \|\nabla\varphi (X_k^{\rm pen})\|_F^2
    \; \le\; 6\sum_i (s_i-1)^2
    \; =\; 6\, \|X_k^{\rm pen}-\overline{ X_k^{\rm pen}}\|_F^2.
$$
Next, we use the fact that projecting onto the Stiefel manifold cannot increase distance:
$$
    \|X_{k+1}^{\rm pen}-\overline{X_{k+1}^{\rm pen}}\|^2
    \; \le\; \bigl\|X_{k+1}^{\rm pen}-\overline{ X_k^{\rm pen}}\bigr\|^2.
$$
By the gradient-step update
$$
    X_{k+1}^{\rm pen}
    = X_k^{\rm pen} - \tfrac13\, \nabla\varphi (X_k^{\rm pen}),
$$
we expand
$$
    \begin{aligned}
    \bigl\|X_{k+1}^{\rm pen}-\overline{ X_k^{\rm pen}}\bigr\|^2
    &= \Bigl\|X_k^{\rm pen}-\tfrac13\nabla\varphi (X_k^{\rm pen})-\overline{ X_k^{\rm pen}}\Bigr\|^2\\
    &= \|X_k^{\rm pen}-\overline{ X_k^{\rm pen}}\|^2
    - \tfrac{2}{3}\bigl\langle X_k^{\rm pen}-\overline{ X_k^{\rm pen}}, \, \nabla\varphi (X_k^{\rm pen})\bigr\rangle
    + \tfrac{1}{9}\|\nabla\varphi (X_k^{\rm pen})\|^2.
    \end{aligned}
$$
Now apply the two key inequalities we proved earlier:
\begin{enumerate}
    \item From Lemma \ref{lem:rsi}, $\langle\nabla\varphi (X_k^{\rm pen}), \, X_k^{\rm pen}-\overline{ X_k^{\rm pen}}\rangle \ge \tfrac32\|X_k^{\rm pen}-\overline{ X_k^{\rm pen}}\|^2$.
   \item From the singular-value bound, $\|\nabla\varphi (X_k^{\rm pen})\|^2\le6\|X_k^{\rm pen}-\overline{ X_k^{\rm pen}}\|^2$.

\end{enumerate}
Substituting gives
$$
    \begin{aligned}
    \|X_{k+1}^{\rm pen}-\overline{X_{k+1}^{\rm pen}}\|^2
    &\le \|X_k^{\rm pen}-\overline{ X_k^{\rm pen}}\|^2
    - \tfrac{2}{3}\cdot\tfrac{3}{2}\, \|X_k^{\rm pen}-\overline{ X_k^{\rm pen}}\|^2
    + \tfrac{1}{9}\cdot6\, \|X_k^{\rm pen}-\overline{ X_k^{\rm pen}}\|^2\\
    &= \bigl (1 - 1 + \tfrac{6}{9}\bigr)\|X_k^{\rm pen}-\overline{ X_k^{\rm pen}}\|^2
    = \tfrac23\, \|X_k^{\rm pen}-\overline{ X_k^{\rm pen}}\|^2.
    \end{aligned}
$$
This completes the proof.
\end{proof}

\subsection{Proof of Lemma \ref{lem:lip}}

\begin{proof}%[Proof of Lemma \ref{lem:lip}]

First inequality \eqref{eq:qub}: Due to the Lipschitz continuity of $f$ and the compactness of ${\rm St}(d,r)$, the inequality \eqref{eq:qub} directly follows from \cite[Lemma 2.4]{chen2021decentralized} and \cite[Lemma 4.2]{deng2023decentralized}, where $L:= L_f + D_f$ with $L_f$ being the Lipschitz constant of $\nabla f_i(X)$ over ${\rm St}(d,r)$ and $D_f:= \max_i\max_{X \in {\rm St}(d,r)} \|\nabla f_i(X)\|$. 
    
Second inequality: We want to bound $\bigl\|\grad f (X) - \hat\nabla f (Y)\bigr\|.$ Recall that $\hat{\nabla} f (X)=\mathcal{P}_{T_X \mathrm{St} (d, r)}\left (\nabla f (X)\right)$ and $\hat{\nabla} f (X)=\operatorname{grad} f (X)$ whenever $X \in \operatorname{St} (d, r)$. We split the difference as
$$
    \begin{aligned}
    \bigl\|\grad f (X) - \hat\nabla f (Y)\bigr\|
    &\le \bigl\|\Pcal_{T_{X}{\rm St} (d, r)} (\nabla f (X)) - \Pcal_{T_{X}{\rm St} (d, r)} (\nabla f (Y))\| \\
    &\quad + \bigl\| \Pcal_{T_{X}{\rm St} (d, r)} (\nabla f (Y)) - \Pcal_{T_{Y}{\rm St} (d, r)} (\nabla f (Y))\bigr\|.
    \end{aligned}
$$
\begin{enumerate}
    \item  The first term is bounded by the contractive property of $\Pcal_{T_{X}{\rm St}(d,r)}$ and the Lipschitz continuity of $\nabla f$:
$$
    \bigl\|\Pcal_{T_{X}{\rm St} (d, r)} (\nabla f (X)) - \Pcal_{T_{X}{\rm St} (d, r)} (\nabla f (Y))\bigr\| \; \le\;  \bigl\|\nabla f (X) - \nabla f (Y)\bigr\|
    \; \le\; L_f\, \|X - Y\|.
$$
\item Recall that $\mathcal{P}_{T_{X} \operatorname{St}(d, r)}(g):=g-X \operatorname{sym}\left(X^{\top} g\right)=g-\tfrac{1}{2}X\left(X^{\top} g+g^{\top} X\right)$. Thus, a direct calculation of second term shows
$$
\begin{aligned}
& \bigl\| \Pcal_{T_{X}{\rm St} (d, r)} (\nabla f (Y)) - \Pcal_{T_{Y}{\rm St} (d, r)} (\nabla f (Y))\bigr\| \\
 = &  \frac{1}{2} \Bigl\|X\bigl (X^\top\nabla f (Y)+\nabla f (Y)^\top X\bigr)
    - Y\bigl (Y^\top\nabla f (Y)+\nabla f (Y)^\top Y\bigr)\Bigr\| \\
 \leq &\frac{1}{2}\| X ( (X-Y)^\top \nabla f (Y) + \nabla f (Y)^\top (X-Y))\| + \frac{1}{2}\| (X-Y) (Y^\top \nabla f (Y) + \nabla f (Y)^\top Y) \| \\
\leq  & \frac{1}{2} (2\hat{D}_f + 3\hat{D}_f) \|X -Y\| \\
 = &  \tfrac52\, \hat D_f\, \|X - Y\|,
\end{aligned}
$$
where $\hat{D}_f:= \max_{Y \in \bar{U}_{{\rm St}(d,r)}(\frac{1}{8})} \|\nabla f(Y)\|$, and the last inequality is from the fact that $\|Y\|_2 = \sigma_{\max }(Y) \leq \frac{3}{2}$. Putting both pieces together,
$$
    \|\grad f (X) - \hat\nabla f (Y)\|
    \; \le\; \Bigl (L_f + \tfrac52\, \hat D_f\Bigr)\, \|X - Y\|.
$$
Setting $\hat L = L_f + \tfrac52\, \hat D_f$ completes the proof.
\end{enumerate}

\end{proof}

\subsection{Proof of Lemma \ref{lem:err-feasi}}

\begin{proof}%[Proof of Lemma \ref{lem:err-feasi}]
It follows from $\bar{X}_{k+1} = \argmin_{X \in {\rm St}(d,r)} \|X - X_{k+1}\|^2$  that 
$$
\begin{aligned}
\|X_{k+1} - \bar{X}_{k+1} \| & \leq \| X_{k+1} - \bar{X}_{k} \| \\
& \leq \| X_k - \alpha_k \Pcal_{T_{X_k} \operatorname{St} (d, r)} (g_k) - \frac{1}{3} \nabla \varphi (X_k) - \bar{X}_{k} \| \quad \text{ (by update \eqref{eq:grad-it}, $\mu = \frac{1}{3}$) }\\
& \leq \| X_{k} - \frac{1}{3} \nabla \varphi (X_k) - \bar{X}_{k} \| + \frac{5}{2}\alpha_k \| g_k\|,
\end{aligned}
$$
where we used the fact 
\begin{equation}
\label{eq:lip-proj-1} 
\| \Pcal_{T_{X_k} {\rm St}(d,r)}(g) \| \leq \|g\| + \|X\|_2^2 \|g\| \leq \frac{5}{2}\|g\|. 
\end{equation}
% $\| \mathcal{P}_{T_{X_k} \operatorname{St}(d, r)}\left(g_k\right)\| \leq \frac{3}{2} \| g_k\|.$ 
According to the proof of Lemma \ref{lem:linear-feasi}, we have
$$
\|X_k - \frac{1}{3}\nabla \varphi (X_k) - \bar{X}_k \|^2 \leq \frac{2}{3}\|X_k - \bar{X}_k\|^2.
$$
Hence, we conclude that $\|X_{k+1} - \bar{X}_{k+1} \| \leq \sqrt{\frac{2}{3}} \| X_k - \bar{X}_k \| + \frac{5}{2}\alpha_k \| g_k\|.$
We complete the proof.
\end{proof}

\subsection{Proof of Lemma \ref{lem:descent}}

\begin{proof}%[Proof of Lemma \ref{lem:descent}]
First, let us prove the following equality  
$$
\iprod{\hat{\nabla} f_i (X)}{\nabla \phi (X)} = \iprod{\nabla f_i (X)}{\Pcal_{T_X {\rm St} (d, r)} (\nabla \phi (X))}.
$$
In fact, using the definition of $\iprod{A}{B} = {\rm tr} (A^\top B)$, we have
$$
\begin{aligned}
\iprod{\hat{\nabla} f_i (X)}{\nabla \phi (X)} = & \iprod{\nabla f_i (X) - X{\rm sym} (X^\top \nabla f_i (X))}{\nabla \phi (X)} \\
= & \iprod{\nabla f_i (X)}{\nabla \phi (X)} - \iprod{X{\rm sym} (X^\top \nabla f_i (X))}{\nabla \phi (X)} \\
= & \iprod{\nabla f_i (X)}{\nabla \phi (X)} - \iprod{{\rm sym} (X^\top \nabla f_i (X))}{X^\top \nabla \phi (X)} \\
= & \iprod{\nabla f_i (X)}{\nabla \phi (X)} - \iprod{X^\top \nabla f_i (X)}{{\rm sym} (X^\top \nabla \phi (X))} \\
= & \iprod{\nabla f_i (X)}{\nabla \phi (X)} - \iprod{\nabla f_i (X)}{X{\rm sym} (X^\top \nabla \phi (X))} \\
= & \iprod{\nabla f_i (X)}{\Pcal_{T_X {\rm St} (d, r)} (\nabla \phi (X))}.
\end{aligned}
$$
Then, it follows from \eqref{eq:qub} that
\begin{align*} 
 f(\bar{X}_{k+1}) - f(\bar{X}_k) \leq  & \iprod{\grad f(\bar{X}_k)}{\bar{X}_{k+1} - \bar{X}_k} + \frac{L}{2}\| \bar{X}_{k+1} - \bar{X}_k \|^2 \\
\leq & \iprod{\grad f(\bar{X}_k)}{\bar{X}_{k+1} - X_{k+1} + X_k- \bar{X}_k} + \iprod{\grad f(\bar{X}_k)}{X_{k+1} - X_k}  + 2L \| X_{k+1} - X_k \|^2 \\
\leq &  \iprod{\grad f(\bar{X}_k)}{\bar{X}_{k+1} - X_{k+1}} + \iprod{\grad f(\bar{X}_k)}{X_{k+1} - X_k}  + 4L (\alpha_k^2 \|\Pcal_{T_{X_k} {\rm St}(d,r)}(g_k) \|^2 + \mu^2 \|\nabla \varphi(X_k)\|^2) \\
= &  \iprod{\grad f(\bar{X}_k) - \grad f(\bar{X}_{k+1}) }{\bar{X}_{k+1} - X_{k+1}} + \iprod{\hat{\nabla} f(X_k) }{X_{k+1} - X_k} \\
& + \iprod{\grad f(\bar{X}_k) - \hat{\nabla} f(X_k)}{X_{k+1} - X_k}  + 4L (\alpha_k^2 \|\Pcal_{T_{X_k} {\rm St}(d,r)}(g_k)\|^2 + \mu^2 \|\nabla \varphi(X_k)\|^2) \\
\leq & 2\hat{L}^2 \|X_{k+1} - X_k\|^2 + \frac{1}{2} \|X_{k+1} - \bar{X}_{k+1}\|^2 - \alpha_k \iprod{\hat{\nabla} f(X_k)}{\Pcal_{T_{X_k} {\rm St}(d,r)}(g_k)} \\
& - \mu \iprod{\hat{\nabla} f(X_k)}{\nabla \varphi(X_k)} + \frac{1}{2}(\hat{L}^2\|X_k -\bar{X}_k\|^2 + \|X_{k+1} - X_k\|^2) \\
& + 4L (\alpha_k^2 \|\Pcal_{T_{X_k} {\rm St}(d,r)}(g_k)\|^2 + \mu^2 \|\nabla \varphi(X_k)\|^2),
\end{align*} 
where the second inequality is from the 2-Lipschitz continuity of $\Pcal_{{\rm St}(d,r)}$ over $\bar{U}_{{\rm St}(d,r)}(\frac{1}{8})$, the last inequality is due to the facts that $X_k - \bar{X}_k \in N_{\bar{X}_k} {\rm St}(d,r)$ and $\iprod{A}{B} \leq \frac{1}{2}(\|A\|^2 + \|B\|^2)$ for any $A, B \in \R^{d \times r}$.
Noticing that 
$$
\begin{aligned}
\|\calP_{T_{X_k} {\rm St} (d, r)} (\nabla \varphi (X_k)) \| & = \|X_k (X_k^\top X_k - I)^2\| \leq \| U_k S_k V_k^\top (V_k S_k^2 V_k^\top - I) V_k \| \\
& \leq \|S_k (S_k+ I)^2\| \|X_k - \bar{X}_k\|^2 \leq 6 \|X_k - \bar{X}_k\|^2,
\end{aligned}
$$
we have
\begin{equation}
\label{eq:descent-ex} 
\begin{aligned}
 \mathbb{E}_k[f(\bar{X}_{k+1})] - f(\bar{X}_k)
\leq & - \alpha_k \|\hat{\nabla} f(X_k)\|^2 - \mu \iprod{\nabla f(X_k)}{\calP_{T_{X_k} {\rm St}(d,r)} (\nabla \varphi(X_k))} + \frac{1}{2}\mathbb{E}_k\|X_{k+1} - \bar{X}_{k+1}\|^2 \\
& + \frac{1}{2}\|X_{k} - \bar{X}_{k}\|^2 + (4\hat{L}^2 + 4L + 1) ( \alpha_k^2 \mathbb{E}_k [\| \Pcal_{T_{X_k} {\rm St}(d,r)}(g_k) \|^2] + \mu^2 \|\nabla \varphi(X_k)\|^2) \\
\leq & - (\alpha_k - (4\hat{L}^2 + 4L + 1) \alpha_k^2) \|\hat{\nabla} f(X_k)\|^2 + \frac{1}{2}\mathbb{E}_k\|X_{k+1} - \bar{X}_{k+1}\|^2 \\
& + (6 \mu \hat{D}_f + \frac{1}{2} \hat{L}^2 + 6(4\hat{L}^2 + 4L + 1)\mu^2) \|X_k - \bar{X}_k\|^2 + 7 (4\hat{L}^2 + 4L + 1) \alpha_k^2 \sigma^2,
\end{aligned} 
\end{equation}
where we use $\mathbb{E}[\Pcal_{T_{X_k} {\rm St}(d,r)}(g_k)] = \Pcal_{T_{X_k} {\rm St}(d,r)}(\mathbb{E}[g_k]) = \hat{\nabla} f(X_k)$ and $\mathbb{E}[\|\Pcal_{T_{X_k} {\rm St}(d,r)}(g_k)\|^2] \leq \|\hat{\nabla} f(X_k)\|^2 + 7 (\mathbb{E}[\|g_k - \hat{\nabla} f(X_k)\|^2]) = \|\hat{\nabla} f(X_k)\|^2 + 7\sigma^2$ from Assumption \ref{assum} and \eqref{eq:lip-proj-1}. Plugging $\mu = \frac{1}{3}$ into \eqref{eq:descent-ex} gives \eqref{eq:descent}.
\end{proof}

\subsection{Proof of Theorem \ref{thm}}

\begin{proof}%[Proof of Theorem \ref{thm}]
First, we show $X_k \in \bar{U}_{{\rm St} (n,d)}(\frac{1}{8})$ for any $k \geq 0$ if $\alpha_k \leq \frac{1}{120 \hat{D}_f}$. In fact, by proof of induction, we have from \eqref{eq:err-feasi} that 
$$
\|X_{k+1} - \bar{X}_{k+1}\| \leq \sqrt{\frac{2}{3}}\|X_k - \bar{X}_k\| + \frac{5}{240 \hat{D}_f} \| g_k \| \leq \frac{1}{8}.
$$
Moreover, applying the standard sequence bound to \eqref{eq:err-feasi} yields
\begin{equation*}
\begin{aligned}
\mathbb{E}\left[\sum_{k=1}^K \|X_k - \bar{X}_{k}\|^2\right] \leq  60  \mathbb{E}\left[\sum_{k=1}^K \alpha_k^2 \| \Pcal_{T_{X_k} {\rm St}(d,r)}(g_k) \|^2 \right] + 4 
\leq  60  
\sum_{k=1}^K \left( \alpha_k^2 \|\hat{\nabla} f(X_k)\|^2 + 7 \alpha_k^2 \sigma^2 \right) + 4
\end{aligned}
\end{equation*}
Then, summing \eqref{eq:descent} over $k=1, \ldots, K$ gives 
\begin{equation*}
\begin{aligned}
    & \mathbb{E}[f(\bar{X}_{K+1})] - f(\bar{X}_1) \\
    \leq & - (\alpha_k - (4\hat{L}^2 + 4L + 1) \alpha_k^2) \mathbb{E}\left[\sum_{k=1}^K \|\hat{\nabla} f(X_k)\|^2\right] \\
    & + \frac{1}{2}\left(4\hat{D}_f + 9\hat{L}^2 + 8 L + 4 \right) \mathbb{E}\left[\sum_{k=1}^{K+1} \|X_k - \bar{X}_k\|^2\right] \\
    \leq & - (\alpha_k - (4\hat{L}^2 + 4L + 1) \alpha_k^2 +  30(4\hat{D}_f + 9\hat{L}^2 + 8 L + 4 )\alpha_k^2 ) \mathbb{E}\left[ \sum_{k=1}^K \|\hat{\nabla} f(X_k)\|^2 \right] \\
    & + \frac{1}{2} \left(4\hat{D}_f + 9\hat{L}^2 + 8 L + 4 \right)(60 \alpha_{K+1}^2 \hat{D}_f^2 + 420 \sigma^2 \sum_{k=1}^{K+1} \alpha_k^2 + 4).
\end{aligned}
\end{equation*}
Define $c_1 = 275 \hat{L}^2 + 244 L + 120 \hat{D}_f + 121$ and $c_2(k) = (30 \alpha_{K+1}^2 \hat{D}_f^2 + 210 \sigma^2 \sum_{k=1}^{K+1} \alpha_k^2 + 2)(4\hat{D}_f + 8\hat{L}^2 + 8 L + 4)$. Then, we have
\begin{equation}
\label{eq:grad-sum} \mathbb{E}\left[\sum_{k=1}^K \alpha_k (1 - c_1 \alpha_k) \|\hat{\nabla} f(X_k)\|^2 \right] \leq f(\bar{X}_1) - f_{\min} + c_2(k), 
\end{equation}
where $f_{\min}:= \argmin_{X \in {\rm St}(d,r)} \; f(X)$. 
Therefore, we have the following iteration complexity results. 
\begin{itemize}
\item If taking constant step sizes $\alpha_k \equiv \alpha \in (0, \frac{1}{2c_1}]$ (which also gives $\alpha \leq \frac{1}{120 \hat{D}_f}$), then
$$
\begin{aligned}
\frac{1}{K} \mathbb{E} \left[ \sum_{k=1}^K \|\hat{\nabla} f (X_k)\|^2 \right] \leq & \frac{2 (f (\bar{X}_1) - f_{\min}) + 24c_3 }{\alpha K} + 420\alpha \sigma^2, \\
\frac{1}{K} \mathbb{E} \left[ \sum_{k=1}^K \|X_k - \bar{X}_k\|^2 \right]\leq & \frac{120\alpha^2 (f (\bar{X}_1) - f_{\min}) + 1440 \alpha^2 c_3 + 4 \alpha }{\alpha K} + 420\alpha^2 (60 \alpha + 1) \sigma^2,
\end{aligned}
$$
where $c_3: = \hat{D}_f + 2\hat{L}^2 + 2L + 1$. 
If in addition the full gradient is used, i.e., $g_k = \nabla f (x_k)$, we have $\sigma = 0$ and
$$
\frac{1}{K} \mathbb{E} \left[ \sum_{k=1}^K \left (\|\hat{\nabla} f (X_k)\|^2 + \|X_k - \bar{X}_k\|^2 \right) \right] = \mathcal{O} \left(\frac{1}{K}\right).
$$
\item If taking step sizes $\alpha_k = \frac{\alpha_0}{\sqrt{k}}$ with $\alpha_0 \in (0, \frac{1}{2c_1}]$, it follows from \eqref{eq:err-feasi} and $\|\hat{\nabla} f (X_k)\| \leq \frac{5}{2} \hat{D}_f$ that
$$ \min_{k \leq K} \|X_k - \bar{X}_k\|^2 \leq 
\frac{1}{K} \mathbb{E}\left[\sum_{k=1}^K \|X_k - \bar{X}_k\|^2 \right] = \mathcal{O} \left(\frac{\log K}{K}\right).
$$
Besides, summing \eqref{eq:err-feasi} and \eqref{eq:grad-sum} gives
$$
\mathbb{E}\left[\sum_{k=1}^K \left (\alpha_k (1 - c_1 \alpha_k) \|\hat{\nabla} f (X_k)\|^2 + \|X_k - \bar{X}_k\|^2 \right) \right] = \mathcal{O} \left(\log K\right).
$$
Furthermore, dividing both sides of \eqref{eq:grad-sum} by $\sum_{k=1} \alpha_k$ leads to
$$
\min_{k \leq K} \; \; \mathbb{E}\left[\|\hat{\nabla} f (X_k)\|^2 + \|X_k - \bar{X}_k\|^2 \right] = \mathcal{O} \left(\frac{\log K}{\sqrt{K}}\right).
$$
\end{itemize}
We complete the proof.

% Therefore, for any $\alpha \leq \frac{1}{2c_1}$ (which implies $\alpha \leq \frac{1}{45 \hat{D}_f}$), taking $K \rightarrow \infty$ gives $\sum_{k=1}^\infty \|\grad f(X_k)\|^2 < \infty$. Then by \eqref{eq:complexity}, $\sum_{k=1}^\infty \|X_k - \bar{X}_{k}\|^2 < \infty$. These lead to \eqref{eq:complexity}. 
\end{proof}

\subsection{Proof of Corollary \ref{coro}}

\begin{proof}%[Proof of Corollary \ref{coro}]
To prove the corollary from Theorem~\ref{thm}, we concatenate $(A, B)$ into a single variable $X$, and verify that all assumptions and conclusions used in Theorem~\ref{thm} remain valid. For simplicity, we assume $B \in {\rm St}(d,r)$. The oblique manifold ${\rm Ob}(d,r)$ can be seen as the product of $r$ Stiefel manifolds ${\rm St}(d,1)$, and the results can be easily extended to the setting. 
The detailed verifications are as follows:

\begin{itemize}
\item \textbf{Lemma~\ref{lem:rsi}}: Since the constraint set is now $\mathbb{R}^{d \times r} \times \mathrm{St} (d, r)$, we define $\bar{X}: = (A, \bar{B})$ and let $\varphi (X) = \frac{1}{4}\|B^\top B - I\|^2$. With this definition, the same conclusion as in the original lemma holds.
\item \textbf{Lemma~\ref{lem:linear-feasi}}: This lemma follows naturally under the new definition of $\bar{X}$, as it depends only on the conclusion of Lemma~\ref{lem:rsi}.
\item \textbf{Lemma~\ref{lem:lip}}: A similar Lipschitz-type result holds by leveraging the Lipschitz continuity of $\nabla \mathcal{L}$ and the geometric properties of the Stiefel manifold associated with $B$.
\item \textbf{Lemma~\ref{lem:err-feasi}}: This is a direct consequence of Lemma~\ref{lem:linear-feasi} and does not require further modification.
\item \textbf{Lemma~\ref{lem:descent}}: By using all the results above, we can obtain a descent inequality for $\mathcal{L} (\bar{X}_{k+1}) - \mathcal{L} (\bar{X}_k)$ in terms of the gradient norm and feasibility errors. Specifically, the descent is characterized by
$$
\|\hat{\nabla} \mathcal{L} (X_k)\|, \quad \|B_k - \bar{B}_k\|^2, \quad \|B_{k+1} - \bar{B}_{k+1}\|^2, \quad \text{and} \quad \sigma^2,
$$
where
$$
\hat{\nabla} \mathcal{L} (X_k) = \left[\nabla_A \mathcal{L} (B_k A_k), \, \hat{\nabla}_B \mathcal{L} (B_k A_k)\right].
$$
\end{itemize}

Therefore, by the chosen step size $\alpha_k$, we can invoke Theorem~\ref{thm} and complete the proof of the corollary.
\end{proof}

\section{Convergence of an Adam-style retraction-free update}
\label{sec:adam-retraction-free}
\label{app:adam-proof}

The convergence of Algorithm~\ref{alg:manlora} with the AdamW update is
difficult to guarantee directly, as even the Euclidean Adam method may fail to
converge without suitable modifications~\cite{reddi2018convergence}.
Nevertheless, our retraction-free idea can be combined with Adam-type correction
techniques from Euclidean optimization to obtain convergence guarantees. As an
example, following the decorrelation idea in AdaShift~\cite{zhou2019adashift},
we use a delayed adaptive preconditioner: the preconditioner is fixed before the
current stochastic gradient is drawn, which makes the alignment argument
conditionally clean.

We first introduce the Adam-style retraction-free update. Given a stochastic
Euclidean gradient $g_k$, define its projected stochastic gradient by
$\tilde g_k:=P_{T_{X_k}\mathrm{St}(d,r)}(g_k)$. Let $m_0=0$, $v_0=0$, and
$\bar v_0=0$. For $k\ge 1$, define
\begin{align}
    m_k
    &=
    \beta_{1,k}m_{k-1}+(1-\beta_{1,k})\tilde g_k,
    \label{eq:adam-moment-m}
    \\
    v_k
    &=
    \beta_2v_{k-1}+(1-\beta_2)(\tilde g_k\odot \tilde g_k),
    \qquad
    \bar v_k=\max\{\bar v_{k-1},v_k\},
    \label{eq:adam-moment-v}
\end{align}
where the maximum is taken componentwise. We identify matrices with their
vectorizations when applying diagonal preconditioners, and define
\[
    H_{k-1}
    :=
    \operatorname{Diag}
    \left(\frac{1}{\sqrt{\bar v_{k-1}}+\epsilon}\right),
    \qquad
    d_k:=H_{k-1}m_k .
\]
The Adam-style retraction-free update is
\begin{equation}
\label{eq:adam-update}
    X_{k+1}
    =
    X_k-\alpha_k d_k-\frac13\nabla\varphi(X_k).
\end{equation}
For the notation, let $\mathcal F_k$ denote the filtration generated by the
history before drawing the stochastic gradient at iteration $k$, and let
$\mathbb E_k[\cdot]:=\mathbb E[\cdot\mid \mathcal F_k]$. For a matrix $X$ in a
neighborhood of $\mathrm{St}(d,r)$, let $\bar X$ denote its projection onto
$\mathrm{St}(d,r)$, and define $\delta(X):=\|X-\bar X\|$. For the iterate
$X_k$, we write $\bar X_k$ for its polar projection and set
$\delta_k:=\|X_k-\bar X_k\|$. We also use the notation in the main paper,
\[
    \hat{\nabla} f(X)
    :=
    \nabla f(X)-X\operatorname{sym}(X^\top\nabla f(X)).
\]
\begin{cond}[Stochasticity for the Adam-style update]
\label{cond:adam-stochasticity}
For every $k$, conditioned on $\mathcal F_k$,
\[
    \mathbb E_k[\tilde g_k]=\hat{\nabla} f(X_k),
    \qquad
    \mathbb E_k\|\tilde g_k-\hat{\nabla} f(X_k)\|^2\le \sigma^2 .
\]
Moreover, there exist constants $G>0$ and $G_\infty>0$ such that
$\|\tilde g_k\|\le G$ and $\|\tilde g_k\|_\infty\le G_\infty$ almost surely.
Finally, $0\le \beta_{1,k}\le \bar\beta<1$ and $0\le \beta_2<1$.
\end{cond}

Set
\[
    h_+:=\frac1\epsilon,
    \qquad
    h_-:=\frac1{G_\infty+\epsilon},
    \qquad
    q:=\sqrt{\frac23}.
\]
Under Condition~\ref{cond:adam-stochasticity}, we have
$h_-I\preceq H_k\preceq h_+I$ for every $k$.

\begin{lem}[Boundedness of the Adam direction]
\label{lem:adam-direction-bound}
Under Condition~\ref{cond:adam-stochasticity}, for every
$k\ge 1$, $\|m_k\|\le G$ and $\|d_k\|\le h_+G$. Moreover,
\begin{equation}
\label{eq:adam-direction-second-moment}
    \mathbb E_k\|d_k\|^2
    \le
    2h_+^2\bigl(\|\hat{\nabla} f(X_k)\|^2+\sigma^2\bigr)
    +8h_+^2\beta_{1,k}^2G^2 .
\end{equation}
\end{lem}

\begin{proof}
The bound $\|m_k\|\le G$ follows by induction from
\eqref{eq:adam-moment-m}, because $m_k$ is a convex combination of
$m_{k-1}$ and $\tilde g_k$. Since $H_{k-1}\preceq h_+I$, we get
$\|d_k\|\le h_+\|m_k\|\le h_+G$.

For the second-moment bound, write
$m_k=\tilde g_k+\beta_{1,k}(m_{k-1}-\tilde g_k)$. Then
\[
    \|m_k\|^2
    \le
    2\|\tilde g_k\|^2
    +2\beta_{1,k}^2\|m_{k-1}-\tilde g_k\|^2
    \le
    2\|\tilde g_k\|^2+8\beta_{1,k}^2G^2 .
\]
Since $\mathbb E_k[\tilde g_k]=\hat{\nabla} f(X_k)$,
\[
    \mathbb E_k\|\tilde g_k\|^2
    =
    \|\hat{\nabla} f(X_k)\|^2+\mathbb E_k\|\tilde g_k-\hat{\nabla} f(X_k)\|^2
    \le
    \|\hat{\nabla} f(X_k)\|^2+\sigma^2 .
\]
Using $\|d_k\|\le h_+\|m_k\|$ proves
\eqref{eq:adam-direction-second-moment}.
\end{proof}

\begin{lem}[Feasibility recursion]
\label{lem:adam-feasibility}
Suppose $X_k\in \bar U_{\mathrm{St}(d,r)}(1/8)$. Then
\begin{equation}
\label{eq:adam-feasibility-recursion}
    \delta_{k+1}
    \le
    q\delta_k+\alpha_k\|d_k\| .
\end{equation}
Consequently, if $\delta_1\le 1/8$ and
$\alpha_k\le (1-q)/(8h_+G)$ for all $k$, then
$X_k\in \bar U_{\mathrm{St}(d,r)}(1/8)$ for all $k$. Moreover, for any
$K\ge 1$,
\begin{equation}
\label{eq:adam-feasibility-sum-d}
    \sum_{k=1}^K\mathbb E[\delta_k^2]
    \le
    6\delta_1^2
    +60\sum_{k=1}^{K-1}\alpha_k^2\mathbb E\|d_k\|^2 .
\end{equation}
In particular,
\begin{align}
    \sum_{k=1}^K\mathbb E[\delta_k^2]
    &\le
    6\delta_1^2
    +120h_+^2\sum_{k=1}^{K-1}\alpha_k^2\mathbb E\|\hat{\nabla} f(X_k)\|^2
    +120h_+^2\sigma^2\sum_{k=1}^{K-1}\alpha_k^2
    \notag\\
    &\quad
    +480h_+^2G^2
    \sum_{k=1}^{K-1}\alpha_k^2\beta_{1,k}^2 .
    \label{eq:adam-feasibility-sum-y}
\end{align}
\end{lem}

\begin{proof}
By the definition of $\bar X_{k+1}$ and the update
\eqref{eq:adam-update},
\begin{align*}
    \delta_{k+1}
    &=
    \|X_{k+1}-\bar X_{k+1}\|
    \le
    \|X_{k+1}-\bar X_k\|                                      \\
    &\le
    \left\|X_k-\frac13\nabla\varphi(X_k)-\bar X_k\right\|
    +\alpha_k\|d_k\|
    \le
    q\delta_k+\alpha_k\|d_k\|,
\end{align*}
where the last inequality follows from the proof of
Lemma~\ref{lem:linear-feasi}. Since $\|d_k\|\le h_+G$, the step-size
condition gives $\delta_{k+1}\le q/8+(1-q)/8=1/8$ whenever
$\delta_k\le 1/8$. Thus the iterates remain in
$\bar U_{\mathrm{St}(d,r)}(1/8)$ by induction.

Iterating \eqref{eq:adam-feasibility-recursion} yields
\[
    \delta_k
    \le
    q^{k-1}\delta_1
    +
    \sum_{j=1}^{k-1}q^{k-1-j}\alpha_j\|d_j\| .
\]
Young's convolution inequality gives
\[
    \sum_{k=1}^K\delta_k^2
    \le
    \frac{2}{1-q^2}\delta_1^2
    +
    \frac{2}{(1-q)^2}
    \sum_{k=1}^{K-1}\alpha_k^2\|d_k\|^2 .
\]
Since $2/(1-q^2)=6$ and $2/(1-q)^2<60$, taking expectation proves
\eqref{eq:adam-feasibility-sum-d}. Combining this estimate with
Lemma~\ref{lem:adam-direction-bound} proves \eqref{eq:adam-feasibility-sum-y}.
\end{proof}

\begin{lem}[Adam-style alignment]
\label{lem:adam-alignment}
Under Condition~\ref{cond:adam-stochasticity}, for every
$k\ge 1$,
\begin{equation}
\label{eq:adam-alignment}
    \mathbb E_k\iprod{\hat{\nabla} f(X_k)}{d_k}
    \ge
    a\|\hat{\nabla} f(X_k)\|^2-R_k,
    \qquad
    R_k:=
    \frac{\beta_{1,k}^2h_+^2G^2}{2(1-\bar\beta)h_-},
    \qquad
    a:=\frac{(1-\bar\beta)h_-}{2}.
\end{equation}
\end{lem}

\begin{proof}
Since $H_{k-1}$ and $m_{k-1}$ are $\mathcal F_k$-measurable and
$\mathbb E_k[\tilde g_k]=\hat{\nabla} f(X_k)$,
\begin{align*}
    \mathbb E_k\iprod{\hat{\nabla} f(X_k)}{d_k}
    &=
    \mathbb E_k\iprod{\hat{\nabla} f(X_k)}{H_{k-1}m_k}                         \\
    &=
    (1-\beta_{1,k})\iprod{\hat{\nabla} f(X_k)}{H_{k-1}\hat{\nabla} f(X_k)}
    +\beta_{1,k}\iprod{\hat{\nabla} f(X_k)}{H_{k-1}m_{k-1}} .
\end{align*}
Using $H_{k-1}\succeq h_-I$, $H_{k-1}\preceq h_+I$, and
$\|m_{k-1}\|\le G$, we obtain
\[
    \mathbb E_k\iprod{\hat{\nabla} f(X_k)}{d_k}
    \ge
    (1-\beta_{1,k})h_-\|\hat{\nabla} f(X_k)\|^2
    -\beta_{1,k}h_+G\|\hat{\nabla} f(X_k)\| .
\]
By Young's inequality,
\[
    \beta_{1,k}h_+G\|\hat{\nabla} f(X_k)\|
    \le
    \frac{(1-\beta_{1,k})h_-}{2}\|\hat{\nabla} f(X_k)\|^2
    +
    \frac{\beta_{1,k}^2h_+^2G^2}
         {2(1-\beta_{1,k})h_-}.
\]
Since $\beta_{1,k}\le \bar\beta$, the claim follows.
\end{proof}

\begin{lem}[One-step descent]
\label{lem:adam-one-step}
Suppose $X_k,X_{k+1}\in \bar U_{\mathrm{St}(d,r)}(1/8)$. Then
\begin{align}
    \mathbb E_k[f(\bar X_{k+1})]-f(\bar X_k)
    &\le
    -a\alpha_k\|\hat{\nabla} f(X_k)\|^2
    +\alpha_kR_k
    +\frac12\mathbb E_k[\delta_{k+1}^2]
    +C_\varphi\delta_k^2
    \notag\\
    &\quad
    +C_0\alpha_k^2\mathbb E_k\|d_k\|^2 .
    \label{eq:adam-one-step}
\end{align}
where
\[
    C_0:=4\max\{\hat L^2,\hat D_f^2\}+4L+1,
    \qquad
    C_\varphi:=2\hat D_f+\frac12\hat L^2+\frac23C_0,
\]
$L$ and $\hat L$ are the constants in Lemma~\ref{lem:lip},
$\hat D_f:=\max_{X\in\bar U_{\mathrm{St}(d,r)}(1/8)}\|\nabla f(X)\|$, and
$a$ is defined in Lemma~\ref{lem:adam-alignment}.
\end{lem}

\begin{proof}
Let $\Delta_k:=X_{k+1}-X_k$, $r_k:=X_k-\bar X_k$,
$r_{k+1}:=X_{k+1}-\bar X_{k+1}$, and
$\eta_k:=\operatorname{grad} f(\bar X_k)$. Averaging the quadratic upper bound
in Lemma~\ref{lem:lip} over the component functions gives
\begin{equation}
\label{eq:adam-smooth-step}
    f(\bar X_{k+1})-f(\bar X_k)
    \le
    \iprod{\eta_k}{\bar X_{k+1}-\bar X_k}
    +\frac L2\|\bar X_{k+1}-\bar X_k\|^2 .
\end{equation}
Writing the compact singular value decomposition of $X_k$ as
$U_kS_kV_k^\top$, the polar projection is $\bar X_k=U_kV_k^\top$, and
$r_k=\bar X_kN_k$ with the symmetric matrix
$N_k:=V_k(S_k-I)V_k^\top$. Since
$\eta_k\in T_{\bar X_k}\mathrm{St}(d,r)$, we have
$\iprod{\eta_k}{r_k}=0$. Moreover,
$\bar X_{k+1}-\bar X_k=\Delta_k+r_k-r_{k+1}$, and hence
\[
    \iprod{\eta_k}{\bar X_{k+1}-\bar X_k}
    =
    \iprod{\hat{\nabla} f(X_k)}{\Delta_k}
    +
    \iprod{\eta_k-\hat{\nabla} f(X_k)}{\Delta_k}
    -
    \iprod{\eta_k}{r_{k+1}} .
\]
By Lemma~\ref{lem:lip}, applied with $X=\bar X_k$ and $Y=X_k$ and then
averaged over the component functions,
\[
    \iprod{\eta_k-\hat{\nabla} f(X_k)}{\Delta_k}
    \le
    \frac12\hat L^2\delta_k^2+\frac12\|\Delta_k\|^2 .
\]
Similarly, $r_{k+1}=\bar X_{k+1}N_{k+1}$ for a symmetric matrix
$N_{k+1}$ with $\|N_{k+1}\|=\delta_{k+1}$. Since
$\iprod{\eta_k}{\bar X_kN_{k+1}}=0$, we have
\begin{align*}
    \left|\iprod{\eta_k}{r_{k+1}}\right|
    &=
    \left|
    \iprod{\eta_k}{(\bar X_{k+1}-\bar X_k)N_{k+1}}
    \right|                                                       \\
    &\le
    \|\eta_k\|\,\|\bar X_{k+1}-\bar X_k\|\,\delta_{k+1}.
\end{align*}
Since $X_k,X_{k+1}\in\bar U_{\mathrm{St}(d,r)}(1/8)$, their smallest singular
values are at least $7/8$. The standard perturbation bound for the polar factor
therefore gives
$\|\bar X_{k+1}-\bar X_k\|\le 2\|X_{k+1}-X_k\|=2\|\Delta_k\|$. Together with
$\|\eta_k\|\le \|\nabla f(\bar X_k)\|\le \hat D_f$, this gives
\[
    \left|\iprod{\eta_k}{r_{k+1}}\right|
    \le
    2\hat D_f\|\Delta_k\|\delta_{k+1}
    \le
    2\hat D_f^2\|\Delta_k\|^2+\frac12\delta_{k+1}^2 .
\]
The same projection Lipschitz estimate gives
\[
    \frac L2\|\bar X_{k+1}-\bar X_k\|^2
    \le
    2L\|\Delta_k\|^2 .
\]
Combining these bounds with \eqref{eq:adam-smooth-step} yields
\begin{equation}
\label{eq:adam-descent-preliminary}
    f(\bar X_{k+1})-f(\bar X_k)
    \le
    \iprod{\hat{\nabla} f(X_k)}{\Delta_k}
    +\frac12\delta_{k+1}^2
    +\frac12\hat L^2\delta_k^2
    +\frac{C_0}{2}\|\Delta_k\|^2 .
\end{equation}

From \eqref{eq:adam-update},
$\Delta_k=-\alpha_kd_k-(1/3)\nabla\varphi(X_k)$. Thus
\[
    \iprod{\hat{\nabla} f(X_k)}{\Delta_k}
    =
    -\alpha_k\iprod{\hat{\nabla} f(X_k)}{d_k}
    -
    \frac13\iprod{\hat{\nabla} f(X_k)}{\nabla\varphi(X_k)} .
\]
The linear map $A\mapsto P_{T_{X_k}\mathrm{St}(d,r)}(A)$ is self-adjoint with
respect to the Frobenius inner product. Therefore,
\[
    \iprod{\hat{\nabla} f(X_k)}{\nabla\varphi(X_k)}
    =
    \iprod{\nabla f(X_k)}
          {P_{T_{X_k}\mathrm{St}(d,r)}(\nabla\varphi(X_k))}.
\]
Let $A_k=X_k^\top X_k-I$. Since $A_k$ is symmetric,
\[
    P_{T_{X_k}\mathrm{St}(d,r)}(\nabla\varphi(X_k))
    =
    X_kA_k-X_k\operatorname{sym}(X_k^\top X_kA_k)
    =
    -X_kA_k^2 .
\]
Writing $X_k=U_kS_kV_k^\top$, the proof of Lemma~\ref{lem:rsi} gives
$7/8\le s_i\le 9/8$ for all singular values. Hence
\[
    \|P_{T_{X_k}\mathrm{St}(d,r)}(\nabla\varphi(X_k))\|
    =
    \|U_kS_k(S_k^2-I)^2V_k^\top\|
    \le
    6\delta_k^2 .
\]
Together with the definition of $\hat D_f$, this gives
\[
    -\frac13\iprod{\hat{\nabla} f(X_k)}{\nabla\varphi(X_k)}
    \le
    \frac13\|\nabla f(X_k)\|
    \|P_{T_{X_k}\mathrm{St}(d,r)}(\nabla\varphi(X_k))\|
    \le
    2\hat D_f\delta_k^2 .
\]
Furthermore,
\[
    \|\Delta_k\|^2
    \le
    2\alpha_k^2\|d_k\|^2
    +
    \frac29\|\nabla\varphi(X_k)\|^2
    \le
    2\alpha_k^2\|d_k\|^2+\frac43\delta_k^2 .
\]
Here the last inequality uses the bound
$\|\nabla\varphi(X_k)\|^2\le 6\delta_k^2$ from the proof of
Lemma~\ref{lem:linear-feasi}.
Substituting these estimates into \eqref{eq:adam-descent-preliminary} gives
\[
    f(\bar X_{k+1})-f(\bar X_k)
    \le
    -\alpha_k\iprod{\hat{\nabla} f(X_k)}{d_k}
    +\frac12\delta_{k+1}^2
    +C_\varphi\delta_k^2
    +C_0\alpha_k^2\|d_k\|^2 .
\]
Taking conditional expectation and applying Lemma~\ref{lem:adam-alignment}
proves \eqref{eq:adam-one-step}.
\end{proof}

\begin{thm}[Convergence of the Adam-style retraction-free update]
\label{thm:adam-convergence}
Suppose that Assumption~\ref{assum} and Condition~\ref{cond:adam-stochasticity}
hold. Let $\{X_k\}$ be generated by \eqref{eq:adam-update}. Let
$f_{\min}:=\inf_{X\in\mathrm{St}(d,r)}f(X)>-\infty$. Assume
$\delta_1\le 1/8$ and
\[
    \alpha_k
    \le
    \min\left\{
    \frac{1-q}{8h_+G},
    \frac{a}{2C_\star}
    \right\},
    \qquad
    C_\star:=
    h_+^2\left[2C_0+120\left(C_\varphi+\frac12\right)\right],
\]
where $a$ is defined in Lemma~\ref{lem:adam-alignment}, and $C_0$ and
$C_\varphi$ are defined in Lemma~\ref{lem:adam-one-step}.
Then all iterates stay in $\bar U_{\mathrm{St}(d,r)}(1/8)$, and for every
$K\ge 1$,
\begin{align}
    \sum_{k=1}^K\alpha_k\mathbb E\|\hat{\nabla} f(X_k)\|^2
    &\le
    \frac2a
    \Bigg[
    f(\bar X_1)-f_{\min}
    +6\left(C_\varphi+\frac12\right)\delta_1^2
    +C_\star\sigma^2\sum_{k=1}^K\alpha_k^2
    \notag\\
    &\qquad\qquad
    +C_\beta\sum_{k=1}^K\alpha_k^2\beta_{1,k}^2
    +C_R\sum_{k=1}^K\alpha_k\beta_{1,k}^2
    \Bigg],
    \label{eq:adam-main-stationarity}
\end{align}
where
\[
    C_\beta:=
    h_+^2G^2
    \left[8C_0+480\left(C_\varphi+\frac12\right)\right],
    \qquad
    C_R:=
    \frac{h_+^2G^2}{2(1-\bar\beta)h_-}.
\]
Moreover,
\begin{align}
    \sum_{k=1}^K\mathbb E[\delta_k^2]
    &\le
    6\delta_1^2
    +120h_+^2\sum_{k=1}^{K-1}\alpha_k^2\mathbb E\|\hat{\nabla} f(X_k)\|^2
    +120h_+^2\sigma^2\sum_{k=1}^{K-1}\alpha_k^2
    \notag\\
    &\quad
    +480h_+^2G^2
    \sum_{k=1}^{K-1}\alpha_k^2\beta_{1,k}^2 .
    \label{eq:adam-main-feasibility}
\end{align}
\end{thm}

\begin{proof}
The step-size condition and Lemma~\ref{lem:adam-feasibility} imply that all
iterates remain in $\bar U_{\mathrm{St}(d,r)}(1/8)$. Hence
Lemma~\ref{lem:adam-one-step} applies for $k=1,\ldots,K$. Summing
\eqref{eq:adam-one-step} gives
\begin{align*}
    \mathbb E[f(\bar X_{K+1})]-f(\bar X_1)
    &\le
    -a\sum_{k=1}^K\alpha_k\mathbb E\|\hat{\nabla} f(X_k)\|^2
    +\sum_{k=1}^K\alpha_kR_k                                    \\
    &\quad
    +\frac12\sum_{k=1}^K\mathbb E[\delta_{k+1}^2]
    +C_\varphi\sum_{k=1}^K\mathbb E[\delta_k^2]
    +C_0\sum_{k=1}^K\alpha_k^2\mathbb E\|d_k\|^2 .
\end{align*}
Since $f(\bar X_{K+1})\ge f_{\min}$,
\begin{align}
    a\sum_{k=1}^K\alpha_k\mathbb E\|\hat{\nabla} f(X_k)\|^2
    &\le
    f(\bar X_1)-f_{\min}
    +\sum_{k=1}^K\alpha_kR_k
    +\left(C_\varphi+\frac12\right)
     \sum_{k=1}^{K+1}\mathbb E[\delta_k^2]
    \notag\\
    &\quad
    +C_0\sum_{k=1}^K\alpha_k^2\mathbb E\|d_k\|^2 .
    \label{eq:adam-proof-before-absorb}
\end{align}
By Lemmas~\ref{lem:adam-direction-bound} and \ref{lem:adam-feasibility},
\[
    \mathbb E\|d_k\|^2
    \le
    2h_+^2\mathbb E\|\hat{\nabla} f(X_k)\|^2
    +2h_+^2\sigma^2
    +8h_+^2\beta_{1,k}^2G^2,
\]
and
\begin{align*}
    \sum_{k=1}^{K+1}\mathbb E[\delta_k^2]
    &\le
    6\delta_1^2
    +120h_+^2\sum_{k=1}^{K}\alpha_k^2\mathbb E\|\hat{\nabla} f(X_k)\|^2
    +120h_+^2\sigma^2\sum_{k=1}^{K}\alpha_k^2                  \\
    &\quad
    +480h_+^2G^2\sum_{k=1}^{K}\alpha_k^2\beta_{1,k}^2 .
\end{align*}
Substituting these two estimates into \eqref{eq:adam-proof-before-absorb}
yields
\begin{align*}
    a\sum_{k=1}^K\alpha_k\mathbb E\|\hat{\nabla} f(X_k)\|^2
    &\le
    f(\bar X_1)-f_{\min}
    +6\left(C_\varphi+\frac12\right)\delta_1^2
    +C_\star\sum_{k=1}^K\alpha_k^2\mathbb E\|\hat{\nabla} f(X_k)\|^2             \\
    &\quad
    +C_\star\sigma^2\sum_{k=1}^K\alpha_k^2
    +C_\beta\sum_{k=1}^K\alpha_k^2\beta_{1,k}^2
    +\sum_{k=1}^K\alpha_kR_k .
\end{align*}
Since $\alpha_k\le a/(2C_\star)$, we have
$C_\star\alpha_k^2\le (a/2)\alpha_k$. Moving this term to the left-hand side
gives
\[
    \frac a2\sum_{k=1}^K\alpha_k\mathbb E\|\hat{\nabla} f(X_k)\|^2
    \le
    f(\bar X_1)-f_{\min}
    +6\left(C_\varphi+\frac12\right)\delta_1^2
    +C_\star\sigma^2\sum_{k=1}^K\alpha_k^2
    +C_\beta\sum_{k=1}^K\alpha_k^2\beta_{1,k}^2
    +\sum_{k=1}^K\alpha_kR_k .
\]
Finally, $R_k=C_R\beta_{1,k}^2$, which proves
\eqref{eq:adam-main-stationarity}. The feasibility estimate
\eqref{eq:adam-main-feasibility} is exactly
\eqref{eq:adam-feasibility-sum-y}.
\end{proof}

\begin{cor}[Iteration complexity]
\label{cor:adam-complexity}
Suppose the assumptions of Theorem~\ref{thm:adam-convergence} hold.

\emph{(i)} If $\alpha_k\equiv \alpha$ and
\[
    \alpha
    \le
    \min\left\{
    1,
    \frac{1-q}{8h_+G},
    \frac{a}{2C_\star}
    \right\},
\]
then
\[
    \frac1K\sum_{k=1}^K\mathbb E\|\hat{\nabla} f(X_k)\|^2
    =
    O\left(
    \frac1{\alpha K}
    +\alpha\sigma^2
    +\frac1K\sum_{k=1}^K\beta_{1,k}^2
    \right).
\]
If, in addition, $\beta_{1,k}=\beta_1\lambda^{k-1}$ with
$\lambda\in(0,1)$, then
\[
    \frac1K\sum_{k=1}^K\mathbb E\|\hat{\nabla} f(X_k)\|^2
    =
    O\left(
    \frac1{\alpha K}
    +\alpha\sigma^2
    \right),
\]
and
\[
    \frac1K\sum_{k=1}^K\mathbb E[\delta_k^2]
    =
    O\left(
    \frac1K+\alpha^2\sigma^2
    \right).
\]
In particular, in the full-gradient case $\sigma=0$ and with geometrically
decaying $\beta_{1,k}$,
\[
    \frac1K\sum_{k=1}^K
    \mathbb E\bigl[\|\hat{\nabla} f(X_k)\|^2+\delta_k^2\bigr]
    =
    O\left(\frac1K\right).
\]

\emph{(ii)} If $\alpha_k=\alpha_0/\sqrt{k}$ with
\[
    \alpha_0
    \le
    \min\left\{
    \frac{1-q}{8h_+G},
    \frac{a}{2C_\star}
    \right\},
\]
and if $\beta_{1,k}=\beta_1\lambda^{k-1}$ with $\lambda\in(0,1)$, then
\[
    \min_{1\le k\le K}\mathbb E\|\hat{\nabla} f(X_k)\|^2
    =
    O\left(\frac{\log K}{\sqrt K}\right),
    \qquad
    \frac1K\sum_{k=1}^K\mathbb E[\delta_k^2]
    =
    O\left(\frac{\log K}{K}\right).
\]
Consequently,
\[
    \min_{1\le k\le K}
    \mathbb E\bigl[\|\hat{\nabla} f(X_k)\|^2+\delta_k^2\bigr]
    =
    O\left(\frac{\log K}{\sqrt K}\right).
\]
\end{cor}

\begin{proof}
For constant $\alpha_k\equiv\alpha$, dividing
\eqref{eq:adam-main-stationarity} by $\alpha K$ gives
\[
    \frac1K\sum_{k=1}^K\mathbb E\|\hat{\nabla} f(X_k)\|^2
    =
    O\left(
    \frac1{\alpha K}
    +\alpha\sigma^2
    +\frac{\alpha}{K}\sum_{k=1}^K\beta_{1,k}^2
    +\frac1K\sum_{k=1}^K\beta_{1,k}^2
    \right).
\]
Since $\alpha\le 1$, this implies the stated stationarity estimate. If
$\beta_{1,k}=\beta_1\lambda^{k-1}$, then
$\sum_{k=1}^K\beta_{1,k}^2=O(1)$, and hence the averaged stationarity estimate
reduces to
\[
    \frac1K\sum_{k=1}^K\mathbb E\|\hat{\nabla} f(X_k)\|^2
    =
    O\left(
    \frac1{\alpha K}
    +\alpha\sigma^2
    \right).
\]
Substituting this bound into \eqref{eq:adam-main-feasibility} gives
\[
    \frac1K\sum_{k=1}^K\mathbb E[\delta_k^2]
    =
    O\left(
    \frac1K
    +\alpha^2\cdot\frac1K\sum_{k=1}^K\mathbb E\|\hat{\nabla} f(X_k)\|^2
    +\alpha^2\sigma^2
    +\frac{\alpha^2}{K}\sum_{k=1}^K\beta_{1,k}^2
    \right).
\]
Using $\alpha\le 1$ and the geometrically decaying momentum sequence, this
yields
\[
    \frac1K\sum_{k=1}^K\mathbb E[\delta_k^2]
    =
    O\left(
    \frac1K+\alpha^2\sigma^2
    \right).
\]
The full-gradient statement follows by setting $\sigma=0$.

For $\alpha_k=\alpha_0/\sqrt{k}$, we have
$\sum_{k=1}^K\alpha_k=\Omega(\sqrt K)$ and
$\sum_{k=1}^K\alpha_k^2=O(\log K)$. If $\beta_{1,k}$ decays geometrically, then
\[
    \sum_{k=1}^K\alpha_k^2\beta_{1,k}^2=O(1),
    \qquad
    \sum_{k=1}^K\alpha_k\beta_{1,k}^2=O(1).
\]
Dividing \eqref{eq:adam-main-stationarity} by
$\sum_{k=1}^K\alpha_k$ gives
\[
    \min_{1\le k\le K}\mathbb E\|\hat{\nabla} f(X_k)\|^2
    \le
    \frac{
    \sum_{k=1}^K\alpha_k\mathbb E\|\hat{\nabla} f(X_k)\|^2
    }{
    \sum_{k=1}^K\alpha_k
    }
    =
    O\left(\frac{\log K}{\sqrt K}\right).
\]
Moreover, since $\alpha_k^2\le \alpha_0\alpha_k$, substituting the same step
sizes into \eqref{eq:adam-main-feasibility} yields
\[
    \sum_{k=1}^K\mathbb E[\delta_k^2]
    =
    O(\log K),
\]
and hence
\[
    \frac1K\sum_{k=1}^K\mathbb E[\delta_k^2]
    =
    O\left(\frac{\log K}{K}\right).
\]
Finally, since $\alpha_k\le \alpha_0$,
\[
    \sum_{k=1}^K\alpha_k\mathbb E[\delta_k^2]
    \le
    \alpha_0\sum_{k=1}^K\mathbb E[\delta_k^2]
    =
    O(\log K).
\]
Combining this estimate with the weighted stationarity estimate and dividing
by $\sum_{k=1}^K\alpha_k=\Omega(\sqrt K)$ gives
\[
    \min_{1\le k\le K}
    \mathbb E\bigl[\|\hat{\nabla} f(X_k)\|^2+\delta_k^2\bigr]
    =
    O\left(\frac{\log K}{\sqrt K}\right).
\]
\end{proof}

\begin{rmk}[Bias correction and standard Adam notation]
\label{rem:adam-bias-correction}
If bias-corrected moments are used while keeping the preconditioner delayed,
for example by replacing $m_k$ and $\bar v_{k-1}$ with their corresponding
bias-corrected versions in the definition of $d_k$, then the same proof applies
after replacing $h_+$, $h_-$, and $G$ by enlarged constants. Thus bias
correction affects only constants and not the convergence rates. The
non-delayed version that uses the current adaptive preconditioner depending on
$\tilde g_k$ is not covered by the conditional-independence argument above.
\end{rmk}

\begin{rmk}[Constant first-momentum parameter]
\label{rem:adam-constant-beta}
If $\beta_{1,k}\equiv \beta_1$ is kept constant, then the same proof yields,
for constant step size,
\[
    \frac1K\sum_{k=1}^K\mathbb E\|\hat{\nabla} f(X_k)\|^2
    =
    O\left(
    \frac1{\alpha K}
    +\alpha\sigma^2
    +\beta_1^2
    \right).
\]
Thus, under this proof strategy, a non-vanishing first-momentum parameter leads
to convergence to a neighborhood. To obtain convergence to a stationary point,
one may use a decaying momentum sequence, for example
$\beta_{1,k}=\beta_1\lambda^{k-1}$ with $\lambda\in(0,1)$.
\end{rmk}

\end{document}